\pdfoutput=1

\documentclass[10pt,a4paper,oneside]{article}
\usepackage[hmargin={2.0 cm, 2.0cm}, top=2.0cm, marginpar=3.5cm]{geometry}
\usepackage[utf8]{inputenc}

\usepackage{amsmath,amsfonts}
\usepackage{algorithmic}
\usepackage{algorithm}
\usepackage{array}
\usepackage[caption=false,font=normalsize,labelfont=sf,textfont=sf]{subfig}
\usepackage{textcomp}
\usepackage{stfloats}
\usepackage{url}
\usepackage{verbatim}
\usepackage{graphicx}
\usepackage{cite}
\date{}
\usepackage{amsthm}

\newtheorem{theor}{Theorem}

\begin{document}

\title{A novel approach for Fair Principal Component Analysis based on eigendecomposition}

\author{Guilherme Dean Pelegrina, Leonardo Tomazeli Duarte\footnote{G. Pelegrina and L. Duarte are with the School of Applied Sciences (FCA), University of Campinas (UNICAMP), Limeira, Brazil. E-mail: guidean@unicamp.br, leonardo.duarte@fca.unicamp.br.}}

\maketitle

\begin{abstract}

Principal component analysis (PCA), a ubiquitous dimensionality reduction technique in signal processing, searches for a projection matrix that minimizes the mean squared error between the reduced dataset and the original one. Since classical PCA is not tailored to address concerns related to fairness, its application to actual problems may lead to disparity in the reconstruction errors of different groups (e.g., men and women, whites and blacks, etc.), with potentially harmful consequences such as the introduction of bias towards sensitive groups. Although several fair versions of PCA have been proposed recently, there still remains a fundamental gap in the search for algorithms that are simple enough to be deployed in real systems. To address this, we propose a novel PCA algorithm which tackles fairness issues by means of a simple strategy comprising a one-dimensional search which exploits the closed-form solution of PCA. As attested by numerical experiments, the proposal can significantly improve fairness with a very small loss in the overall reconstruction error and without resorting to complex optimization schemes. Moreover, our findings are consistent in several real situations as well as in scenarios with both unbalanced and balanced datasets.

\end{abstract}

\section{Introduction}
\label{sec:intro}

Dimensionality reduction techniques have been used in signal processing and machine learning problems in order to deal with high dimensional datasets and, therefore, to enhance data visualization and reduce the complexity of learning algorithms~\cite{Kaski2011,Huang2019,Reddy2020}. Among such techniques, one of the most used is the Principal Component Analysis (PCA)~\cite{Jolliffe2002}, which has been applied in several problems~\cite{Kang2016,Zhao2019,Feng2019,Chakraborty2020,Rajani2020}. In summary, PCA aims at reducing the dimensionality of a dataset while preserving as much information as possible from this dataset. Besides maximizing the retained information, there is also an interest in the development of methods that avoid bias when providing dimensionality reduction. Indeed, the classical PCA formulation does not take into account different sensitive groups when projecting the data. As a consequence, the reduced dataset may contain distinct representation errors for those different groups, which in turn may lead to bias for certain sensitive groups. Therefore, unless strategies to mitigate bias are deployed, the application of PCA in machine-learning-based systems may suffer from fairness issues.

Several recent studies have addressed bias and fairness issues in real world problems~\cite{Barocas2019,Mhasawade2021,Booth2021,Cheong2021}. For instance, a topic of interest is the trade-off between model accuracy and fairness~\cite{Kleinberg2017,Dressel2018,Haas2019,Zhang2020,Rodolfa2021,Zhang2021}. In PCA-based dimensionality reduction, trade-offs between the quality of representation and fairness have also been studied in the literature~\cite{Pelegrina2021,Kamani2022}. For instance, a central question is how much one is willing to lose in the overall representation error in order to decrease the disparity between the sensitive groups. Aiming at reducing such a disparity, several works in the literature proposed fair PCA-based dimensionality reduction techniques~\cite{Samadi2018,Olfat2019,Tantipongpipat2019,Morgenstern2019,Pelegrina2021,Zalcberg2021,Kamani2022}.

There are basically two main lines of reasoning in these approaches. For instance, in the algorithm FairPCA~\cite{Samadi2018}, as well as in other related approaches~\cite{Kamani2022,Tantipongpipat2019,Morgenstern2019}, fairness is measured by means of the loss suffered by each sensitive group with respect to their individual optimal projection. In the optimal scenario, the proposed algorithm will find a projection matrix that leads to the same average loss for each sensitive group. On the other hand, in the Multi-Objective Fair Principal Component Analysis (MOFPCA) algorithm~\cite{Pelegrina2021}, one simply measures fairness by taking the squared difference between the averaged reconstruction errors of both sensitive groups. Therefore, in this case, the fairest scenario is the one that minimizes the disparity between the groups. Moreover, the formulation in MOFPCA does not require to find $r$ projections vectors that improve fairness. One may exploit the eigenvectors provided by the classical PCA and, in a different combination (i.e., not necessarily the eigenvectors associated with the $r$ highest eigenvalues), evaluate if fairness was improved without a large loss in the overall reconstruction error. The price that MOFPCA pays is that the projection vectors are restricted to the ones provided by the classical PCA.

In this paper, we propose a novel approach to exploit fairness in PCA-based dimensionality reduction. Similarly as in the MOFPCA, we also consider the disparity between the reconstructions errors as a fairness measure. However, we exploit any projection matrix that improve fairness in the reduced space. For this purpose, we formulate an optimization problem whose cost function includes both overall reconstruction error and the adopted fairness measure. In order to cast this cost function as a mono-objective optimization problem, the objectives are weighted by a scalar factor. The interesting aspect is that, given a predefined weighted factor, the solution has a closed form based on the eigenvector/eigenvalue decomposition. Therefore, a first contribution of this paper consists in formulating an optimization problem that provides a compromise solution between the overall reconstruction error and the fairness measure, which depends on the adopted weighted factor and can be solved by an eigendecomposition.

Moreover, since a central concern is reducing the disparity between the sensitive groups, one may investigate which weighted factor lead the the fairest scenario. Therefore, the other contribution of this paper is to propose an efficient fair PCA-based dimensionality reduction algorithm that minimizes the disparity in the averaged reconstruction errors by means of a one-dimensional search in closed form solutions. Note the existing fair PCA methods either require a complex mono-optimization scheme or a multiobjective scheme which is limited to the principal components which stem from the classical PCA algorithm.

The rest of this paper is organized as follows. Section~\ref{sec:dispPCA} describes the classical PCA formulation and the possible disparities in dimensionality reduction tasks. In Section~\ref{sec:propos}, we present the proposed approach for fair dimensionality reduction. The numerical experiments as well as the obtained results are discussed in Section~\ref{sec:exp}. Finally, our conclusions and future perspectives are presented in Section~\ref{sec:concl}.

\section{Disparities in principal component analysis}
\label{sec:dispPCA}

Let $\mathbf{X} \in \mathbb{R}^{n \times d}$ denote a dataset with $n$ samples and $d$ attributes. For convenience, let us also assume that $\mathbf{X}$ has zero mean (otherwise, one should center the data by extracting its mean). With the purpose of reducing the dimensionality of $\mathbf{X}$ from $d$ to $r$-dimensional samples, the goal in PCA is to find a projection matrix $\mathbf{U} \in \mathbb{R}^{d \times r}$ such that the projected data $\mathbf{X}\mathbf{U}$ are uncorrelated and retain as much information from $\mathbf{X}$ as possible. A typical solution of PCA is obtained by tackling the following optimization problem:
\begin{equation}
\begin{array}{ll}
    \underset{\mathbf{U}}{\min} & \left\|\mathbf{X} - \mathbf{X}\mathbf{U}\mathbf{U}^T\right\|_{F}^{2} \\
		\text{s.t.} & \mathbf{U}^T\mathbf{U} = \mathbf{I}
\end{array},
\label{eq:classical_PCA}
\end{equation}
where $\mathcal{R}_{\mathbf{X}}(\mathbf{U}) = \left\|\mathbf{X} - \mathbf{XUU}^T\right\|_{F}^{2}$ represents the overall reconstruction error, $\left\| \cdot \right\|_{F}$ is the Frobenius norm~\cite{Golub2013} and $\mathbf{U}^T\mathbf{U} = \mathbf{I}$, where $\mathbf{I}$ is the identity matrix, ensures that $\mathbf{U}$ is orthogonal. It is easy to show that minimizing $ \left\|\mathbf{X} - \mathbf{XUU}^T\right\|_{F}^{2}$ is equivalent to maximize $\text{tr}\left( \mathbf{X}\mathbf{X}^T \right) - \text{tr}\left( \mathbf{U}^T\mathbf{X}^T\mathbf{X}\mathbf{U} \right)$ (or maximize $\text{tr}\left( \mathbf{U}^T\mathbf{C}_{\mathbf{X}}\mathbf{U} \right)$, where $\mathbf{C}_{\mathbf{X}}=\frac{1}{n}\mathbf{X}^T\mathbf{X}$ is the covariance matrix of $\mathbf{X}$, since $\text{tr}\left( \mathbf{X}\mathbf{X}^T \right)$ is a constant)~\cite{Bro2014,Wang2018,Huang2021}. Moreover, since $\text{tr}\left( \mathbf{U}^T\mathbf{X}^T\mathbf{X}\mathbf{U} \right) = \left\|\mathbf{X}\mathbf{U}\right\|_{F}^{2} = \sum_{j=1}^r \mathbf{u}_j^T\mathbf{X}^T\mathbf{X}\mathbf{u}_j$, minimizing the reconstruction error is also equivalent to maximize the total variance of the projected data. It is also known in the literature~\cite{Jolliffe2002} that the solution of PCA is achieved by setting the columns of $\mathbf{U}$ as the eigenvectors of $\mathbf{C}_{\mathbf{X}}$ associated with its highest eigenvalues (see Appendix~\ref{app:pca_sol}).

Although one expects to optimize the overall model performance (i.e., to achieve the minimum reconstruction error in PCA), there is a growing interest in the literature in the development of automatic decision systems that are also fair. We here consider that a fair algorithm should not lead to disparities between sensitive groups (e.g., males and females, whites and blacks, etc.). Therefore, a concern in the classical PCA is that it does not take into account possible disparities between different groups when projecting the dataset. For instance, assume that $\mathbf{X} = \left[\mathbf{X}_A; \mathbf{X}_B \right]$, where $\mathbf{X}_A \in \mathbb{R}^{n_A \times d}$ and $\mathbf{X}_B \in \mathbb{R}^{n_B \times d}$ represent different sensitive groups with $n_A$ and $n_B$ samples, respectively. Without loss of generality, let us also assume that group $A$ is privileged in comparison with group $B$. This means that the application of the classical PCA in $\mathbf{X}$ leads to a reduced dataset in which the average reconstruction error of group $A$ is lower in comparison with group $B$, i.e., $\bar{\mathcal{R}}_{\mathbf{X}_A}(\mathbf{U}) = \frac{1}{n_A}\mathcal{R}_{\mathbf{X}_A}(\mathbf{U}) < \frac{1}{n_B}\mathcal{R}_{\mathbf{X}_B}(\mathbf{U}) = \bar{\mathcal{R}}_{\mathbf{X}_B}(\mathbf{U})$. Therefore, this disparity may harm the representation of group $B$ in comparison with group $A$ in the projected data. In order to illustrate this scenario, let us consider the example presented in Figures~\ref{fig:illustr_pca1},~\ref{fig:illustr_pca2} and~\ref{fig:illustr_pca3}. One may see that, by projecting the dataset into the first principal component, the variance of group $A$ is greater than the variance of group $B$. Therefore, group $B$ has a worst representation in comparison with group $A$.

\begin{figure*}[ht]
\centering
\subfloat[]{\includegraphics[width=2.1in]{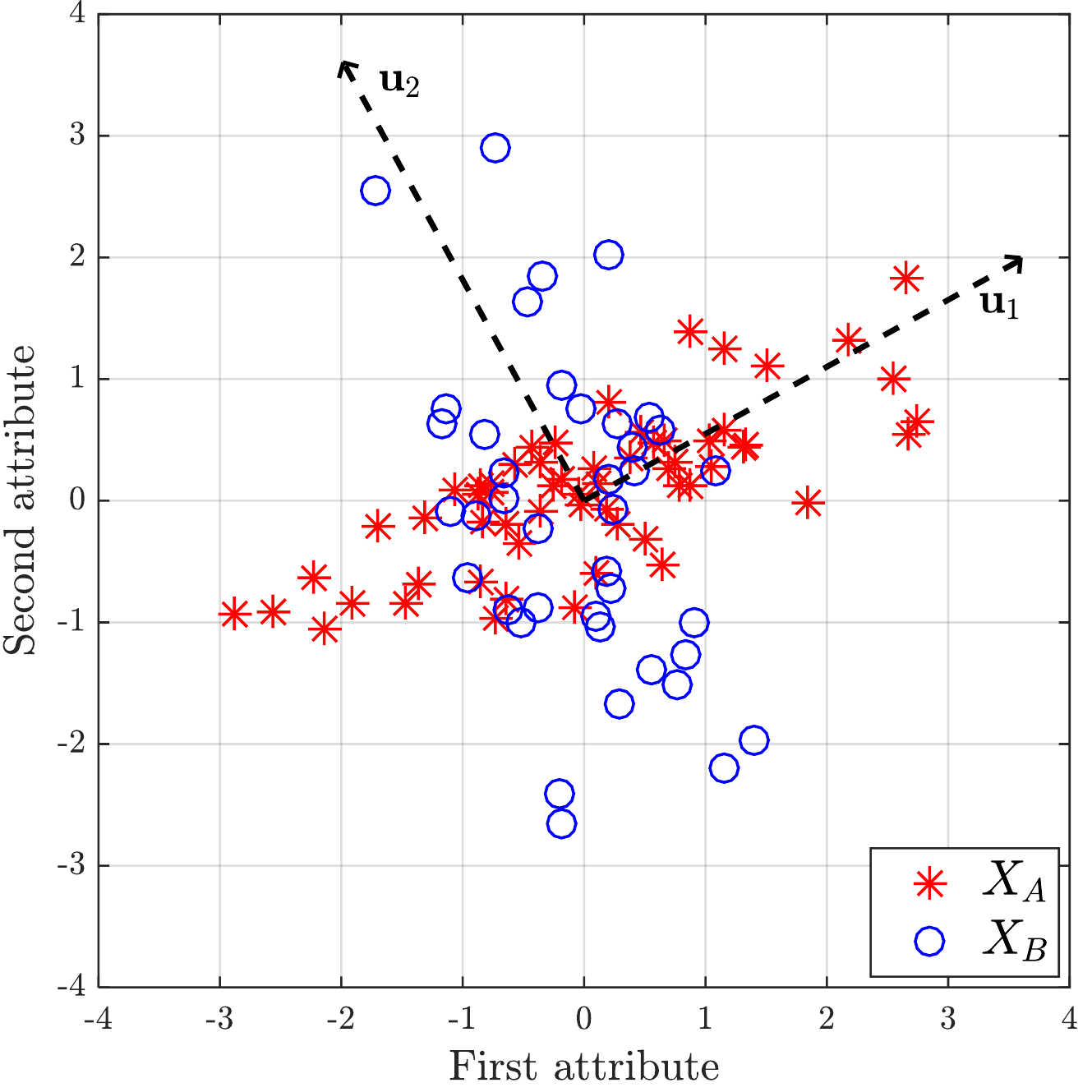}
\label{fig:illustr_pca1}}
\hfil
\subfloat[]{\includegraphics[width=2.1in]{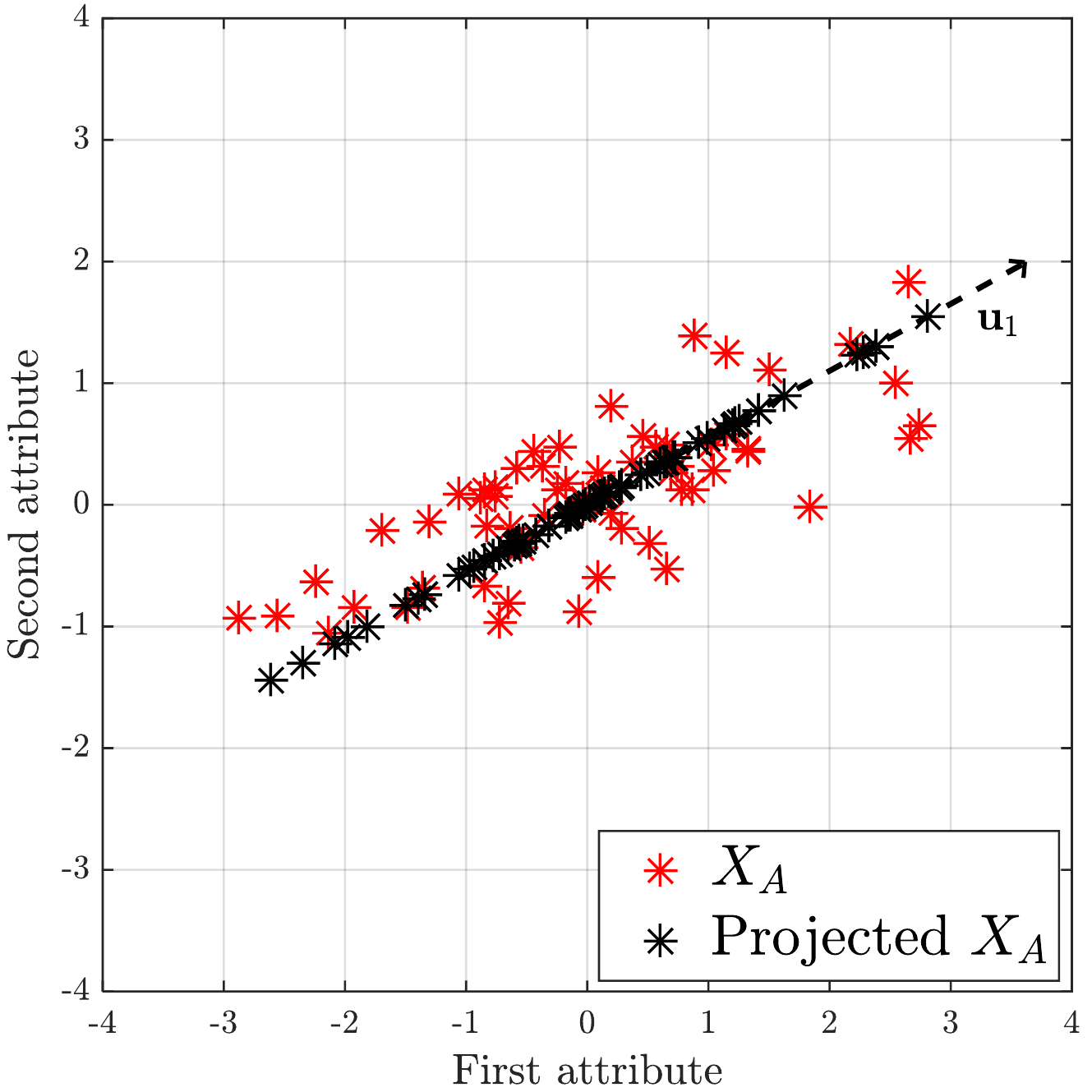}
\label{fig:illustr_pca2}}
\hfil
\subfloat[]{\includegraphics[width=2.1in]{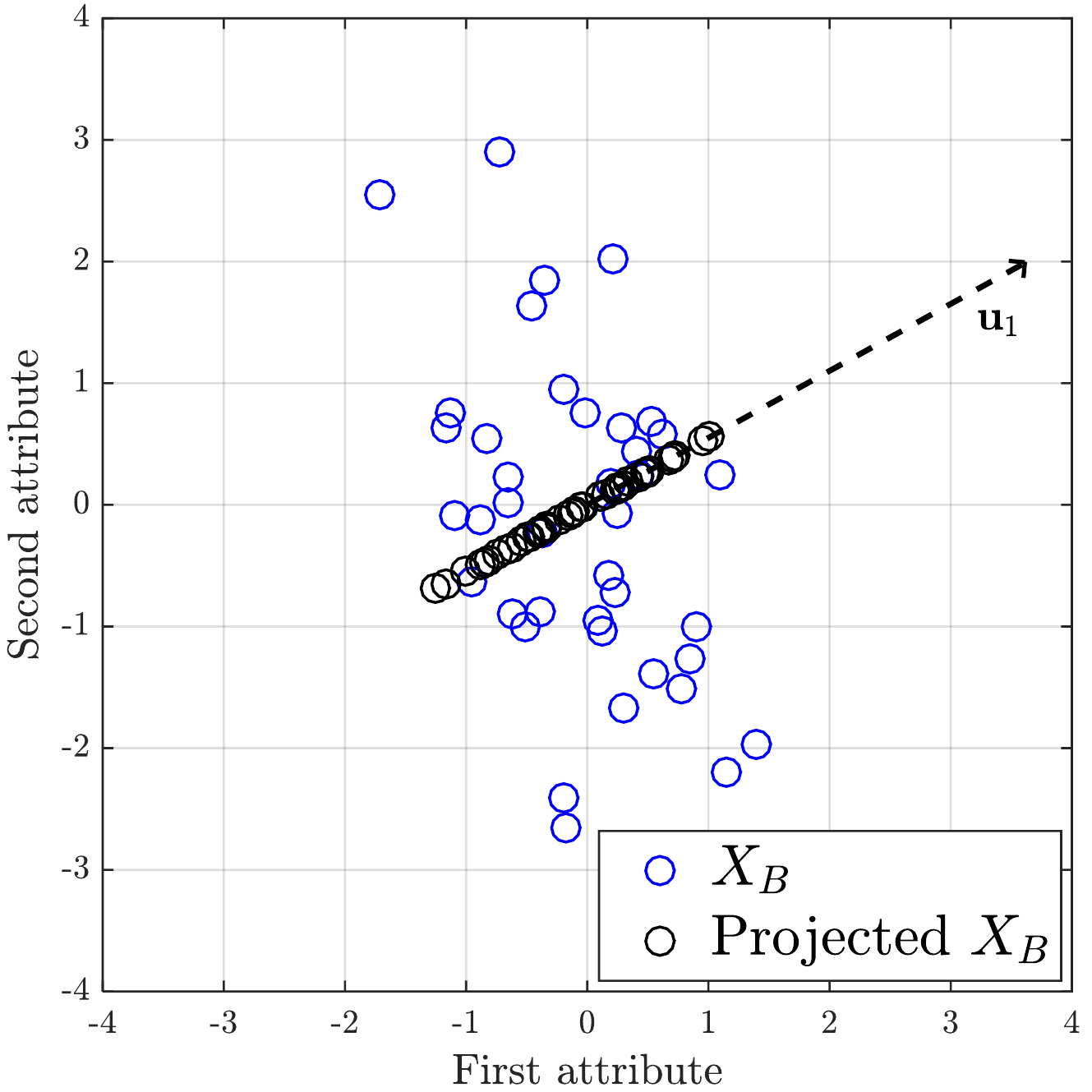}
\label{fig:illustr_pca3}}
\hfill
\subfloat[]{\includegraphics[width=2.1in]{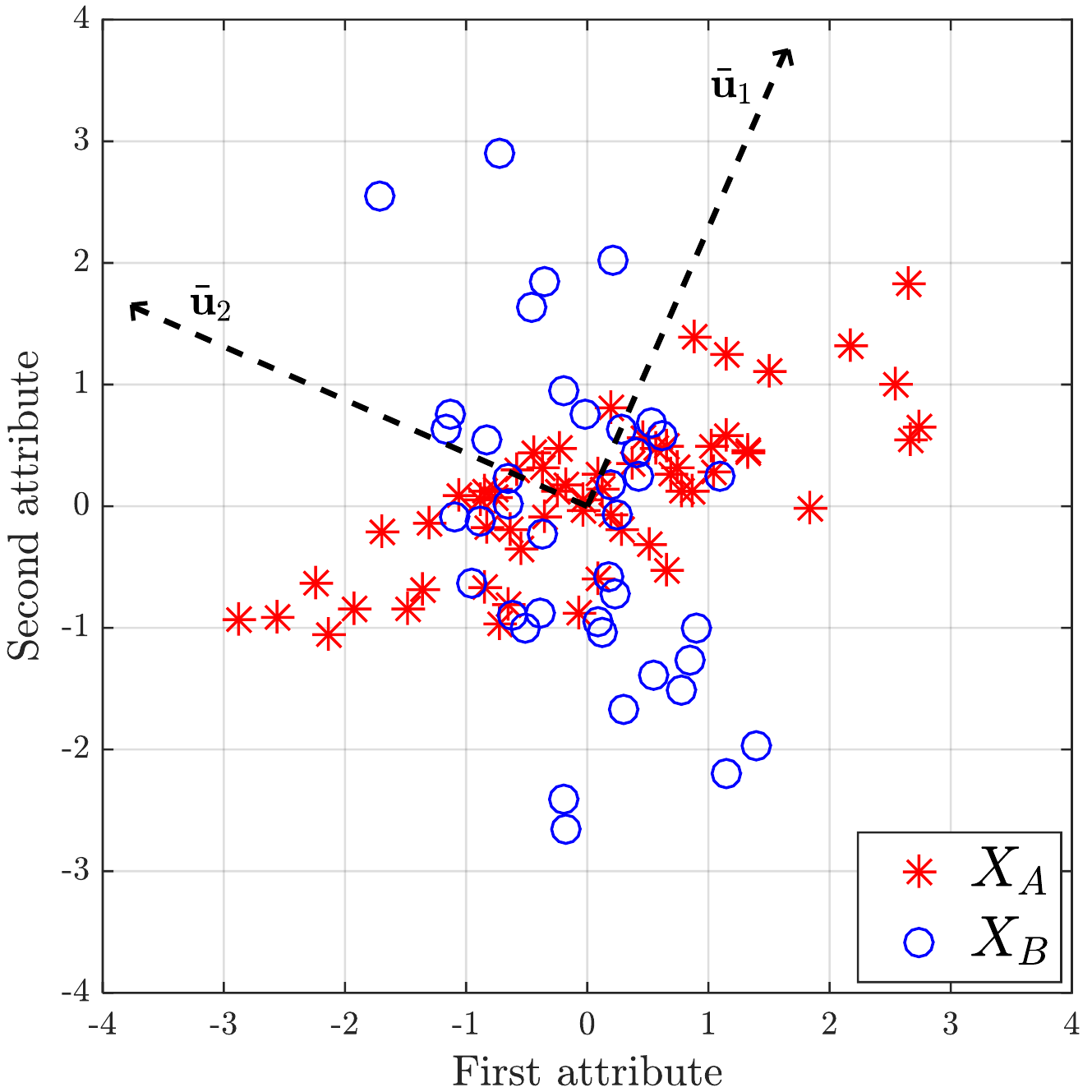}
\label{fig:illustr_fpca1}}
\hfil
\subfloat[]{\includegraphics[width=2.1in]{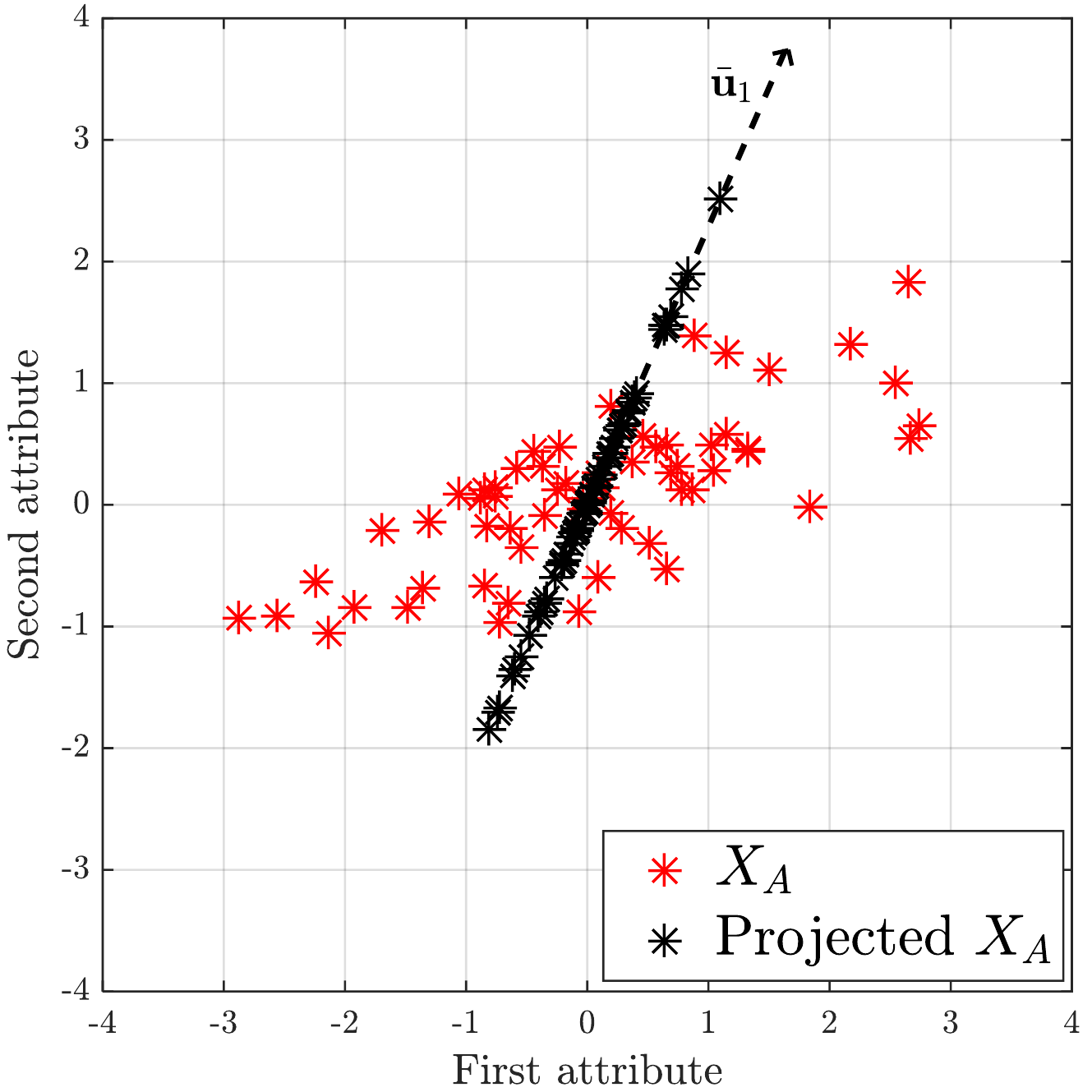}
\label{fig:illustr_fpca2}}
\hfil
\subfloat[]{\includegraphics[width=2.1in]{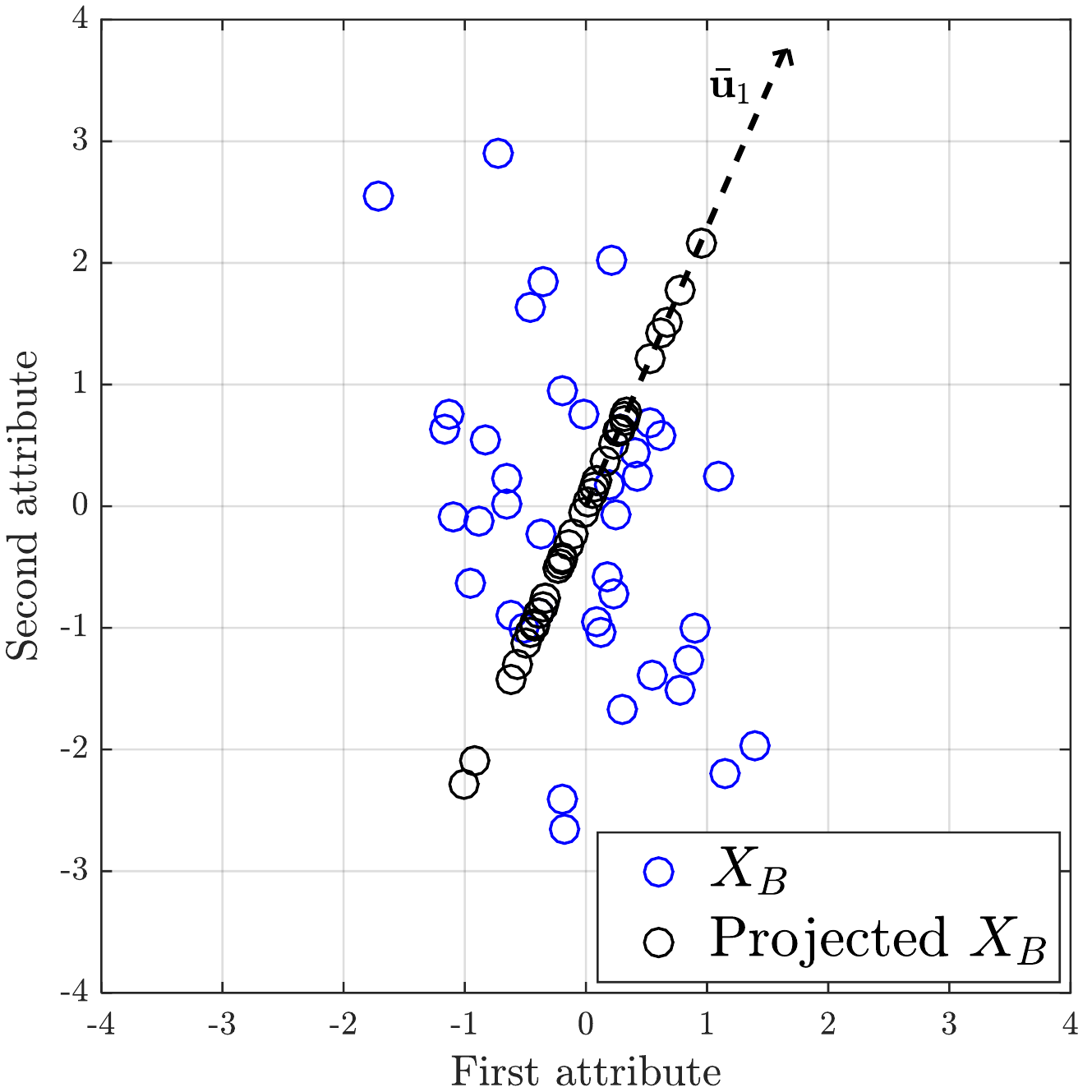}
\label{fig:illustr_fpca3}}
\caption{Illustrative situation in which PCA generates disparities between different groups.~\protect\subref{fig:illustr_pca1} Scatter plot of dataset $\mathbf{X}$ and the projection vectors $\mathbf{u}_1, \mathbf{u}_2$ obtained by the classical PCA.~\protect\subref{fig:illustr_pca2} Projection of group $A$ into $\mathbf{u}_1$.~\protect\subref{fig:illustr_pca3} Projection of group $B$ into $\mathbf{u}_1$.~\protect\subref{fig:illustr_fpca1} Scatter plot of dataset $\mathbf{X}$ and the projection vectors $\bar{\mathbf{u}}_1, \bar{\mathbf{u}}_2$ obtained by our fair PCA-based proposal.~\protect\subref{fig:illustr_fpca2} Fair projection of group $A$ into $\bar{\mathbf{u}}_1$.~\protect\subref{fig:illustr_fpca3}, Fair projection of group $B$ into $\bar{\mathbf{u}}_1$.}
\label{fig:illustr_disp}
\end{figure*}

In order to avoid disparities in the reconstruction errors, it is convenient to adopt a disparity measure when performing dimensionality reduction. We here propose the following one:
\begin{equation}
    \mathcal{D}_{\mathbf{X}_B,\mathbf{X}_A}(\mathbf{U}) = \frac{\left\|\mathbf{X}_B - \mathbf{X}_B\mathbf{UU}^T\right\|_{F}^{2}}{n_B} - \frac{\left\|\mathbf{X}_A - \mathbf{X}_A\mathbf{UU}^T\right\|_{F}^{2}}{n_A},
\label{eq:disp_meas}
\end{equation}
which calculates the disparity between the average reconstruction errors of the harmed group ($\bar{\mathcal{R}}_{\mathbf{X}_B}$) and the privileged one ($\bar{\mathcal{R}}_{\mathbf{X}_A}$). For instance, if $\mathcal{D}_{\mathbf{X}_B,\mathbf{X}_A}(\mathbf{U}) = 0$, we achieved the fairest scenario. If one considers the example illustrated in Figures~\ref{fig:illustr_fpca1},~\ref{fig:illustr_fpca2} and~\ref{fig:illustr_fpca3}, one may see that, by assuming another projection vector, the retained information for groups $A$ and $B$ are similar, which indicates that the disparity between their representations was mitigated.

\section{Proposed fair PCA-based dimensionality reduction approach}
\label{sec:propos}

Our formulation for fair principal component analysis consists in minimizing both overall reconstruction error and $\mathcal{D}_{\mathbf{X}_B,\mathbf{X}_A}(\mathbf{U})$. Therefore, we deal with the following optimization problem:
\begin{equation}
\begin{array}{ll}
    \underset{\mathbf{U}}{\min} & \alpha \bar{\mathcal{R}}_{\mathbf{X}}(\mathbf{U}) + \left( 1 - \alpha \right) \mathcal{D}_{\mathbf{X}_B,\mathbf{X}_A}(\mathbf{U})  \\
		\text{s.t.} & \mathbf{U}^T\mathbf{U} = \mathbf{I},
\end{array}
\label{eq:fair_PCA}
\end{equation}
where $\bar{\mathcal{R}}_{\mathbf{X}}(\mathbf{U}) = \frac{1}{n}\mathcal{R}_{\mathbf{X}}(\mathbf{U})$ and $\alpha \in [0,1]$ is the weighting factor that controls the importance given to the each objective. Note that, if $\alpha = 1$,~\eqref{eq:fair_PCA} is equivalent to the classical PCA formulation. However, the lower is the value of $\alpha$, the greater is the importance assigned to the disparity measure. Therefore, different values of $\alpha$ lead to trade-off solutions between the fairness metric and the overall reconstruction error.

Similarly as in the classical PCA, the interesting aspect of our proposal is that the optimal solution can be also achieved by means of an eigendecomposition. This finding is stated in the following theorem: 

\begin{theor}{}
Assume a predefined weighting factor $\alpha$. The optimal solution of the fair principal component analysis problem expressed in~\eqref{eq:fair_PCA} is given by the eigenvectors of the weighed covariance matrices $\mathbf{\hat{C}} = \left( \alpha \frac{ \mathbf{X}^T\mathbf{X}}{n} + \left( 1 - \alpha \right) \left( \frac{ \mathbf{X}_A^T\mathbf{X}_A }{n_A} - \frac{ \mathbf{X}_B^T\mathbf{X}_B }{n_B} \right) \right)$ associated with its highest eigenvalues.
\end{theor}

\begin{proof}
Since $\alpha \bar{\mathcal{R}}_{\mathbf{X}}(\mathbf{U}) + \left( 1 - \alpha \right) \mathcal{D}_{\mathbf{X}_B,\mathbf{X}_A}(\mathbf{U}) =  \alpha \bar{\mathcal{R}}_{\mathbf{X}}(\mathbf{U}) + \left( 1 - \alpha \right) \left( \bar{\mathcal{R}}_{\mathbf{X}_B} - \bar{\mathcal{R}}_{\mathbf{X}_A} \right)$, where $\bar{\mathcal{R}}_{\mathbf{X}}(\mathbf{U}) = \frac{\left\|\mathbf{X} - \mathbf{X}\mathbf{UU}^T\right\|_{F}^{2}}{n} = \frac{ \text{tr}\left( \mathbf{X}\mathbf{X}^T \right) - \text{tr}\left( \mathbf{U}^T\mathbf{X}^T\mathbf{X}\mathbf{U} \right) }{n}$ (and similarly for $\bar{\mathcal{R}}_{\mathbf{X}_A}$ and $\bar{\mathcal{R}}_{\mathbf{X}_B}$), it is easy to show that
\begin{equation}
\begin{array}{l}
      \alpha \bar{\mathcal{R}}_{\mathbf{X}}(\mathbf{U}) + \left( 1 - \alpha \right) \mathcal{D}_{\mathbf{X}_B,\mathbf{X}_A}(\mathbf{U}) \\
      = \kappa - \alpha \frac{\text{tr}\left( \mathbf{U}^T\mathbf{X}^T\mathbf{X}\mathbf{U} \right) }{n} \\
      \, \, \, \, \, + \left( 1 - \alpha \right) \left( \frac{ \text{tr}\left( \mathbf{U}^T\mathbf{X}_A^T\mathbf{X}_A\mathbf{U} \right) }{n_A} - \frac{ \text{tr}\left( \mathbf{U}^T\mathbf{X}_B^T\mathbf{X}_B\mathbf{U} \right) }{n_B} \right)  \\
      = \kappa - \text{tr}\left( \mathbf{U}^T \frac{ \alpha \mathbf{X}^T\mathbf{X}}{n} \mathbf{U} \right) \\
      \, \, \, \, \, + \text{tr}\left( \mathbf{U}^T \frac{ \left( 1 - \alpha \right) \mathbf{X}_A^T\mathbf{X}_A }{n_A} \mathbf{U} \right) - \text{tr}\left( \mathbf{U}^T \frac{ \left( 1 - \alpha \right) \mathbf{X}_B^T\mathbf{X}_B }{n_B} \mathbf{U} \right)  \\
      = \kappa - \text{tr}\left( \mathbf{U}^T \left( \alpha \frac{ \mathbf{X}^T\mathbf{X}}{n} + \left( 1 - \alpha \right) \left( \frac{  \mathbf{X}_A^T\mathbf{X}_A }{n_A} - \frac{ \mathbf{X}_B^T\mathbf{X}_B }{n_B} \right) \right)  \mathbf{U} \right)  \\
      = \kappa - \text{tr}\left( \mathbf{U}^T \mathbf{\hat{C}} \mathbf{U} \right)  \\
\end{array}
\end{equation}
where $\kappa = \alpha \frac{\text{tr}\left( \mathbf{X}\mathbf{X}^T \right)}{n} + \left( 1 - \alpha \right) \left( \frac{ \text{tr}\left( \mathbf{X}_B\mathbf{X}_B^T \right)}{n_B} - \frac{ \text{tr}\left( \mathbf{X}_A\mathbf{X}_A^T \right) )}{n_A} \right)$ and $\mathbf{\hat{C}} = \left( \alpha \frac{ \mathbf{X}^T\mathbf{X}}{n} + \left( 1 - \alpha \right) \left( \frac{ \mathbf{X}_A^T\mathbf{X}_A }{n_A} - \frac{ \mathbf{X}_B^T\mathbf{X}_B }{n_B} \right) \right)$. Since $\kappa$ is a constant, the minimization of $\alpha \bar{\mathcal{R}}_{\mathbf{X}}(\mathbf{U}) + \left( 1 - \alpha \right) \mathcal{D}_{\mathbf{X}_B,\mathbf{X}_A}(\mathbf{U})$ leads to the maximization of $\text{tr}\left( \mathbf{U}^T \mathbf{\hat{C}} \mathbf{U} \right)$. Similarly as in the classical PCA, our proposal also turns to the maximization of a trace operator. Therefore, one may follow the iterative approach to find the columns of $\mathbf{U}$. We start by the first principal component:
\begin{equation}
\label{eq:fpca_pc1}
\begin{array}{ll}
\displaystyle\max_{\mathbf{u}_1} & \mathbf{u}_1^T\mathbf{\hat{C}}\mathbf{u}_1 \\
\text{s.t.} & \mathbf{u}_1^T\mathbf{u}_1 = 1, \\
\end{array}
\end{equation}
This optimization problem is equivalent to the one presented in~\eqref{eq:pca_pc1}. Therefore, we know that $\mathbf{u}_1$ is equal to the eigenvector of $\mathbf{\hat{C}}$ associated with the highest eigenvalue $\lambda_1$. For the next $r-1$ projection vectors, we must include the constraint that ensures that $\mathbf{u}_j^T, \mathbf{u}_{j-1}=0$, for all $j=2, \ldots, r$ and solve the following the following optimization problems:
\begin{equation}
\label{eq:fpca_pc2}
\begin{array}{ll}
\displaystyle\max_{\mathbf{u}_j} & \mathbf{u}_j^T\mathbf{\hat{C}}\mathbf{u}_j \\
\text{s.t.} & \mathbf{u}_j^T\mathbf{u}_j = 1, \\
 & \mathbf{u}_j^T\mathbf{u}_{j-1} = 0, \\
\end{array}, \, \, \, \forall j=2, \ldots, r.
\end{equation}
However, instead of assuming that $\mathbf{u}_j^T\mathbf{C}_{\mathbf{X}}\mathbf{u}_{j-1} = 0$, for all $j=2, \ldots, r$, we here assume that $\mathbf{u}_j^T\mathbf{\hat{C}}\mathbf{u}_{j-1} = 0$, for all $j=2, \ldots, r$. As a consequence, we do not ensure that the projected data are uncorrelated. However, we may follow the derivation of the classical PCA and conclude that the columns of $\mathbf{U}$ are composed, in its columns, by the eigenvectors of $\mathbf{\hat{C}}$ associated with its highest eigenvalues.
\end{proof}

In summary, given a weighting factor $\alpha$, we attempt to reduce the disparities by solving an optimization problem very similar to the classical PCA. However, in our approach, the eigendecomposition of matrix $\mathbf{C}_{\mathbf{X}}$ takes into account the disparity between sensitive groups. Clearly, for $\alpha = 1$, one only minimizes the overall reconstruction error and does not reduce possible disparities (i.e., one may have $\bar{\mathcal{R}}_{\mathbf{X}_A}(\mathbf{U}) < \bar{\mathcal{R}}_{\mathbf{X}_B}(\mathbf{U})$). On the other hand, $\alpha = 0$ may minimize the disparity measure more than the necessary, which may invert the privileged group. In other words, one may achieve a projected data in which group $B$ has a better representation in comparison with group $A$ (i.e., $\bar{\mathcal{R}}_{\mathbf{X}_B}(\mathbf{U}) < \bar{\mathcal{R}}_{\mathbf{X}_A}(\mathbf{U})$). Therefore, it is possible that there is a value of $\alpha$ that leads to $\bar{\mathcal{R}}_{\mathbf{X}_A}(\mathbf{U}) \approx \bar{\mathcal{R}}_{\mathbf{X}_B}(\mathbf{U})$, i.e., that minimizes the considered fairness measure
\begin{equation}
\mathcal{F}_{\mathbf{X}_B,\mathbf{X}_A}(\mathbf{U}) = \left( \bar{\mathcal{R}}_{\mathbf{X}_B}(\mathbf{U}) - \bar{\mathcal{R}}_{\mathbf{X}_A}(\mathbf{U}) \right)^2.
\label{eq:fair_meas}
\end{equation}
Note that~\eqref{eq:fair_meas} is the square of $\mathcal{D}_{\mathbf{X}_B,\mathbf{X}_A}(\mathbf{U})$. Therefore, when minimizing $\mathcal{F}_{\mathbf{X}_B,\mathbf{X}_A}(\mathbf{U})$, one reduces the disparity and avoids the inversion of the privileged group. In the sequel, we present the proposed algorithms for fair dimensionality reduction.

\subsection{Unconstrained fair dimensionality reduction}

Since fixed $\alpha$ the optimization problem has a closed solution, the (unconstrained) fair PCA-based dimensionality reduction algorithm proposed in this paper (called here \textit{u-FPCA}) consists in a one-dimensional search (in $\alpha$) of eigenvectors/eigenvalues solutions aiming at minimizing the fairness measure~\eqref{eq:fair_meas}. Mathematically, this formulation can be expressed as follows:
\begin{equation}
\begin{array}{ll}
    \underset{\alpha}{\min} & \left( \frac{\left\|\mathbf{X}_B - \mathbf{X}_B\mathbf{UU}^T\right\|_{F}^{2}}{n_B} - \frac{\left\|\mathbf{X}_A - \mathbf{X}_A\mathbf{UU}^T\right\|_{F}^{2}}{n_A} \right)^2,
\end{array}
\label{eq:fair_PCA_proposal}
\end{equation}
where the columns of $\mathbf{U} \in \mathbb{R}^{d \times r}$ are composed by the eigenvectors of $\mathbf{\hat{C}}$ associated with its highest eigenvalues. Therefore, it is a simple and efficient approach that can be solved, for instance, by applying a golden section search algorithm~\cite{Vajda1989}. Based on the dataset $\mathbf{X}$ divided into the sensitive groups $G_1$ and $G_2$, the steps of u-FPCA are presented in Algorithm~\ref{alg:ufpca} (in this algorithm, $eig \left(\mathbf{P}, q \right)$ represents the operator that returns the $q$ eigenvectors of $\mathbf{P}$ associated with its $q$ highest eigenvalues). Recall that, without loss of generality, we assume that group $A$ is the privileged one when applying the classical PCA.

\begin{algorithm}[h!t]
    \caption{(\textit{u-FPCA})}
    \label{alg:ufpca}
		\begin{algorithmic}
				\STATE \textbf{Input:} Dataset $\mathbf{X}$, sensitive data $\mathbf{X}_{G_1}$ and $\mathbf{X}_{G_2}$, data sizes $n$, $n_{G_1}$ and $n_{G_2}$, reduced dimension $r$ and tolerance $tol$ for the golden section search.
				
				\STATE \textbf{Output:} Fair projection matrix $\mathbf{U}$.
				
				\STATE 1: \textbf{Compute the covariance matrix of $\mathbf{X}$}: $\mathbf{C}_{\mathbf{X}} \leftarrow \frac{\mathbf{X}^T\mathbf{X}}{n}$
				
				\STATE 2: \textbf{Compute PCA}: $\bar{\mathbf{U}} \leftarrow eig \left(\mathbf{C}_{\mathbf{X}}, r \right)$
				
				\STATE 3: \textbf{Compute the averaged reconstruction errors:} $\bar{\mathcal{R}}_{\mathbf{X}_{G_1}}^{PCA}(\bar{\mathbf{U}}) \leftarrow \frac{\left\|\mathbf{X}_{G_1} - \mathbf{X}_{G_1}\bar{\mathbf{U}}\bar{\mathbf{U}}^T\right\|_{F}^{2}}{n_{G_1}}$ and $\bar{\mathcal{R}}_{\mathbf{X}_{G_2}}^{PCA}(\bar{\mathbf{U}}) \leftarrow \frac{\left\|\mathbf{X}_{G_2} - \mathbf{X}_{G_2}\bar{\mathbf{U}}\bar{\mathbf{U}}^T\right\|_{F}^{2}}{n_{G_2}}$
				
				\STATE 4: \textbf{Define the privileged and harmed groups}:
				\IF{$\bar{\mathcal{R}}_{\mathbf{X}_{G_1}}^{PCA}(\bar{\mathbf{U}}) \leq \bar{\mathcal{R}}_{\mathbf{X}_{G_2}}^{PCA}(\bar{\mathbf{U}})$}
						\STATE $A \leftarrow G_1$, $n_A = n_{G_1}$, $B \leftarrow G_2$ and $n_B = n_{G_2}$
				\ELSE
						\STATE $A \leftarrow G_2$, $n_A = n_{G_2}$, $B \leftarrow G_1$ and $n_B = n_{G_1}$
				\ENDIF
				
				\STATE 5: \textbf{Apply the golden section search}:
				
				\STATE \textbf{Define}: $\alpha_0 \leftarrow 0$, $\alpha_1 \leftarrow 1$ and $g_{ratio} \leftarrow (\sqrt{5} + 1)/2$
				
				\WHILE{$\alpha_1 - \alpha_0 \leq tol$}
				        \STATE \textbf{Compute the candidates for $\alpha$}: $\alpha_{1a} \leftarrow \alpha_0 + (\alpha_1 - \alpha_0)/g_{ratio}$ and $\alpha_{0a} = \alpha_1 - (\alpha_1 - \alpha_0)/g_{ratio}$
				        
				        \STATE \textbf{Compute the weighted covariance matrices}: $\hat{\mathbf{C}}_{1} \leftarrow \alpha_{1a}\mathbf{C}_{\mathbf{X}} + (1-\alpha_{1a})\left( \frac{\mathbf{X}_A^T\mathbf{X}_A}{n_A} - \frac{\mathbf{X}_B^T\mathbf{X}_B}{n_B} \right)$ and $\hat{\mathbf{C}}_{0} \leftarrow \alpha_{0a}\mathbf{C}_{\mathbf{X}} + (1-\alpha_{0a})\left( \frac{\mathbf{X}_A^T\mathbf{X}_A}{n_A} - \frac{\mathbf{X}_B^T\mathbf{X}_B}{n_B} \right)$
				        
				        \STATE \textbf{Compute the candidates for $\mathbf{U}$}: $\mathbf{U}_{1} \leftarrow eig \left(\hat{\mathbf{C}}_{1}, r \right)$ and $\mathbf{U}_{0} \leftarrow eig \left(\hat{\mathbf{C}}_{0}, r \right)$
				        
				        \STATE \textbf{Compute the averaged reconstruction errors}: $\bar{\mathcal{R}}_{\mathbf{X}_{A}}(\mathbf{U}_{1}) \leftarrow \frac{\left\|\mathbf{X}_{A} - \mathbf{X}_{A}\mathbf{U}_{1}\mathbf{U}_{1}^T\right\|_{F}^{2}}{n_{A}}$, $\bar{\mathcal{R}}_{\mathbf{X}_{A}}(\mathbf{U}_{0}) \leftarrow \frac{\left\|\mathbf{X}_{A} - \mathbf{X}_{A}\mathbf{U}_{0}\mathbf{U}_{0}^T\right\|_{F}^{2}}{n_{A}}$, $\bar{\mathcal{R}}_{\mathbf{X}_{B}}(\mathbf{U}_{1}) \leftarrow \frac{\left\|\mathbf{X}_{B} - \mathbf{X}_{B}\mathbf{U}_{1}\mathbf{U}_{1}^T\right\|_{F}^{2}}{n_{B}}$ and $\bar{\mathcal{R}}_{\mathbf{X}_{B}}(\mathbf{U}_{0}) \leftarrow \frac{\left\|\mathbf{X}_{B} - \mathbf{X}_{B}\mathbf{U}_{0}\mathbf{U}_{0}^T\right\|_{F}^{2}}{n_{B}}$
				        
				        \STATE \textbf{Compute the fairness measures}: $\mathcal{F}_{\mathbf{X}_B,\mathbf{X}_A}(\mathbf{U}_{1}) \leftarrow \left( \bar{\mathcal{R}}_{\mathbf{X}_{A}}(\mathbf{U}_{1}) - \bar{\mathcal{R}}_{\mathbf{X}_{B}}(\mathbf{U}_{1}) \right)^2$ and $\mathcal{F}_{\mathbf{X}_B,\mathbf{X}_A}(\mathbf{U}_{0}) \leftarrow \left( \bar{\mathcal{R}}_{\mathbf{X}_{A}}(\mathbf{U}_{0}) - \bar{\mathcal{R}}_{\mathbf{X}_{B}}(\mathbf{U}_{0}) \right)^2$
				        
				        \IF{$\mathcal{F}_{\mathbf{X}_B,\mathbf{X}_A}(\mathbf{U}_{0}) \leq \mathcal{F}_{\mathbf{X}_B,\mathbf{X}_A}(\mathbf{U}_{1})$}
				                \STATE $\alpha_1 \leftarrow \alpha_{1a}$
				        \ELSE
				                \STATE $\alpha_0 \leftarrow \alpha_{0a}$
				        \ENDIF
				\ENDWHILE
				
				\STATE $\alpha \leftarrow (\alpha_1 + \alpha_0)/2$

				\STATE 6: \textbf{Compute the weighted covariance matrix}: $\hat{\mathbf{C}} \leftarrow \alpha \mathbf{C}_{\mathbf{X}} + (1-\alpha)\left( \frac{\mathbf{X}_A^T\mathbf{X}_A}{n_A} - \frac{\mathbf{X}_B^T\mathbf{X}_B}{n_B} \right)$
				
				\STATE 7: \textbf{Compute the fair projection matrix}: $\mathbf{U} \leftarrow eig \left(\hat{\mathbf{C}}, r \right)$
    \end{algorithmic}
\end{algorithm}

\subsection{Constrained fair dimensionality reduction}

Without further considerations, the optimization problem~\eqref{eq:fair_PCA_proposal} will only minimize the fairness measure. In this case, there are two possible scenarios for enhancing fairness, either by degrading both $\bar{\mathcal{R}}_{\mathbf{X}_A}(\mathbf{U})$ and $\approx \bar{\mathcal{R}}_{\mathbf{X}_B}(\mathbf{U})$ (group $A$ more than group $B$) or by degrading $\frac{1}{n_A}\mathcal{R}_A(\mathbf{U})$ while improving $\frac{1}{n_B}\mathcal{R}_B(\mathbf{U})$. We would say that the first scenario may be an inconvenient, since we must reduce the quality of representation of both groups. However, the second scenario may be acceptable, since in order to reduce the disparities between the reconstruction errors, the privileged group will ``resign'' part of its quality in representation in order to improve the harmed one. Therefore, with the purpose of ensuring that one only achieves a solution in the second scenario, one should adopt a constraint in the one-dimensional search. In the second approach proposed in this paper (called here \textit{c-FPCA}), the optimization problem can be formulated as follows:
\begin{equation}
\begin{array}{ll}
    \underset{\alpha}{\min} & \left( \frac{\left\|\mathbf{X}_B - \mathbf{X}_B\mathbf{UU}^T\right\|_{F}^{2}}{n_B} - \frac{\left\|\mathbf{X}_A - \mathbf{X}_A\mathbf{UU}^T\right\|_{F}^{2}}{n_A} \right)^2 \\
    \text{s.t.} & \bar{\mathcal{R}}_k(\mathbf{U}) \leq \bar{\mathcal{R}}_B^{PCA}(\tilde{\mathbf{U}}), \, \, \, \forall k \in \left\{A,B \right\} \\
\end{array}
\label{eq:fair_PCA_proposal_2}
\end{equation}
where the columns of $\mathbf{U} \in \mathbb{R}^{d \times r}$ are the same as before and $\bar{\mathcal{R}}_B^{PCA}(\tilde{\mathbf{U}})$ is the overall reconstruction error of group $B$ (the unprivileged one) by considering the projection matrix $\mathbf{\tilde{U}}$ obtained from the solution of the classical PCA on all dataset $\mathbf{X}$. In order to solve~\eqref{eq:fair_PCA_proposal_2}, one may also use a one-dimensional search algorithm, such as the golden section search, and include the aforementioned constraint. For instance, one may start from $\alpha = 1$ (which is the solution of the classical PCA and, therefore, a feasible solution) and decrease its value until achieve the minimum of the fairness measure while preserving $\bar{\mathcal{R}}_A(\mathbf{U}),\bar{\mathcal{R}}_B(\mathbf{U}) \leq \bar{\mathcal{R}}_B^{PCA}(\tilde{\mathbf{U}})$. Algorithm~\ref{alg:cfpca} presents the steps of the c-FPCA.

\begin{algorithm}[h!t]
    \caption{(\textit{c-FPCA})}
    \label{alg:cfpca}
		\begin{algorithmic}
				\STATE \textbf{Input:} Dataset $\mathbf{X}$, sensitive data $\mathbf{X}_{G_1}$ and $\mathbf{X}_{G_2}$, data sizes $n$, $n_{G_1}$ and $n_{G_2}$, reduced dimension $r$ and tolerance $tol$ for the golden section search.
				
				\STATE \textbf{Output:} Fair projection matrix $\mathbf{U}$.
				
				\STATE 1: \textbf{Compute the covariance matrix of $\mathbf{X}$}: $\mathbf{C}_{\mathbf{X}} \leftarrow \frac{\mathbf{X}^T\mathbf{X}}{n}$
				
				\STATE 2: \textbf{Compute PCA}: $\bar{\mathbf{U}} \leftarrow eig \left(\mathbf{C}_{\mathbf{X}}, r \right)$
				
				\STATE 3: \textbf{Compute the averaged reconstruction errors:} $\bar{\mathcal{R}}_{\mathbf{X}_{G_1}}^{PCA}(\bar{\mathbf{U}}) \leftarrow \frac{\left\|\mathbf{X}_{G_1} - \mathbf{X}_{G_1}\bar{\mathbf{U}}\bar{\mathbf{U}}^T\right\|_{F}^{2}}{n_{G_1}}$ and $\bar{\mathcal{R}}_{\mathbf{X}_{G_2}}^{PCA}(\bar{\mathbf{U}}) \leftarrow \frac{\left\|\mathbf{X}_{G_2} - \mathbf{X}_{G_2}\bar{\mathbf{U}}\bar{\mathbf{U}}^T\right\|_{F}^{2}}{n_{G_2}}$
				
				\STATE 4: \textbf{Define the privileged and harmed groups}:
				\IF{$\bar{\mathcal{R}}_{\mathbf{X}_{G_1}}^{PCA}(\bar{\mathbf{U}}) \leq \bar{\mathcal{R}}_{\mathbf{X}_{G_2}}^{PCA}(\bar{\mathbf{U}})$}
						\STATE $A \leftarrow G_1$, $n_A = n_{G_1}$, $B \leftarrow G_2$ and $n_B = n_{G_2}$
				\ELSE
						\STATE $A \leftarrow G_2$, $n_A = n_{G_2}$, $B \leftarrow G_1$ and $n_B = n_{G_1}$
				\ENDIF
				
				\STATE 5: \textbf{Apply the golden section search}:
				
				\STATE \textbf{Define}: $\alpha_0 \leftarrow 0$, $\alpha_1 \leftarrow 1$ and $g_{ratio} \leftarrow (\sqrt{5} + 1)/2$
				
				\WHILE{$\alpha_1 - \alpha_0 \leq tol$}
				        \STATE \textbf{Compute the candidates for $\alpha$}: $\alpha_{1a} \leftarrow \alpha_0 + (\alpha_1 - \alpha_0)/g_{ratio}$ and $\alpha_{0a} = \alpha_1 - (\alpha_1 - \alpha_0)/g_{ratio}$
				        
				        \STATE \textbf{Compute the weighted covariance matrices}: $\hat{\mathbf{C}}_{1} \leftarrow \alpha_{1a}\mathbf{C}_{\mathbf{X}} + (1-\alpha_{1a})\left( \frac{\mathbf{X}_A^T\mathbf{X}_A}{n_A} - \frac{\mathbf{X}_B^T\mathbf{X}_B}{n_B} \right)$ and $\hat{\mathbf{C}}_{0} \leftarrow \alpha_{0a}\mathbf{C}_{\mathbf{X}} + (1-\alpha_{0a})\left( \frac{\mathbf{X}_A^T\mathbf{X}_A}{n_A} - \frac{\mathbf{X}_B^T\mathbf{X}_B}{n_B} \right)$
				        
				        \STATE \textbf{Compute the candidates for $\mathbf{U}$}: $\mathbf{U}_{1} \leftarrow eig \left(\hat{\mathbf{C}}_{1}, r \right)$ and $\mathbf{U}_{0} \leftarrow eig \left(\hat{\mathbf{C}}_{0}, r \right)$
				        
				        \STATE \textbf{Compute the averaged reconstruction errors}: $\bar{\mathcal{R}}_{\mathbf{X}_{A}}(\mathbf{U}_{1}) \leftarrow \frac{\left\|\mathbf{X}_{A} - \mathbf{X}_{A}\mathbf{U}_{1}\mathbf{U}_{1}^T\right\|_{F}^{2}}{n_{A}}$, $\bar{\mathcal{R}}_{\mathbf{X}_{A}}(\mathbf{U}_{0}) \leftarrow \frac{\left\|\mathbf{X}_{A} - \mathbf{X}_{A}\mathbf{U}_{0}\mathbf{U}_{0}^T\right\|_{F}^{2}}{n_{A}}$, $\bar{\mathcal{R}}_{\mathbf{X}_{B}}(\mathbf{U}_{1}) \leftarrow \frac{\left\|\mathbf{X}_{B} - \mathbf{X}_{B}\mathbf{U}_{1}\mathbf{U}_{1}^T\right\|_{F}^{2}}{n_{B}}$ and $\bar{\mathcal{R}}_{\mathbf{X}_{B}}(\mathbf{U}_{0}) \leftarrow \frac{\left\|\mathbf{X}_{B} - \mathbf{X}_{B}\mathbf{U}_{0}\mathbf{U}_{0}^T\right\|_{F}^{2}}{n_{B}}$
				        
				        \STATE \textbf{Compute the fairness measures}: $\mathcal{F}_{\mathbf{X}_B,\mathbf{X}_A}(\mathbf{U}_{1}) \leftarrow \left( \bar{\mathcal{R}}_{\mathbf{X}_{A}}(\mathbf{U}_{1}) - \bar{\mathcal{R}}_{\mathbf{X}_{B}}(\mathbf{U}_{1}) \right)^2$ and $\mathcal{F}_{\mathbf{X}_B,\mathbf{X}_A}(\mathbf{U}_{0}) \leftarrow \left( \bar{\mathcal{R}}_{\mathbf{X}_{A}}(\mathbf{U}_{0}) - \bar{\mathcal{R}}_{\mathbf{X}_{B}}(\mathbf{U}_{0}) \right)^2$
				        
				        \IF{$\mathcal{F}_{\mathbf{X}_B,\mathbf{X}_A}(\mathbf{U}_{0}) \leq \mathcal{F}_{\mathbf{X}_B,\mathbf{X}_A}(\mathbf{U}_{1})$}
				                \IF{$\bar{\mathcal{R}}_{\mathbf{X}_{A}}(\mathbf{U}_{0}) \leq \bar{\mathcal{R}}_{\mathbf{X}_{B}}^{PCA}(\bar{\mathbf{U}})$ and $\bar{\mathcal{R}}_{\mathbf{X}_{B}}(\mathbf{U}_{0}) \leq \bar{\mathcal{R}}_{\mathbf{X}_{B}}^{PCA}(\bar{\mathbf{U}})$}
				                        \STATE $\alpha_1 \leftarrow \alpha_{1a}$
				                \ELSE
				                        \STATE $\alpha_0 \leftarrow \alpha_{0a}$
				                \ENDIF
    			        \ELSE
				                \STATE $\alpha_0 \leftarrow \alpha_{0a}$
				        \ENDIF
				\ENDWHILE
				
				\STATE $\alpha \leftarrow (\alpha_1 + \alpha_0)/2$

				\STATE 6: \textbf{Compute the weighted covariance matrix}: $\hat{\mathbf{C}} \leftarrow \alpha \mathbf{C}_{\mathbf{X}} + (1-\alpha)\left( \frac{\mathbf{X}_A^T\mathbf{X}_A}{n_A} - \frac{\mathbf{X}_B^T\mathbf{X}_B}{n_B} \right)$
				
				\STATE 7: \textbf{Compute the fair projection matrix}: $\mathbf{U} \leftarrow eig \left(\hat{\mathbf{C}}, r \right)$
    \end{algorithmic}
\end{algorithm}

\section{Experiments}
\label{sec:exp}

The experiments\footnote{It is worth mentioning that all codes supporting this paper are available in the public repository \url{https://github.com/GuilhermePelegrina/FPCA}. The pre-processed data are available at \url{https://github.com/GuilhermePelegrina/Datasets/tree/main/FPCA}.} benchmark our proposals against the classical PCA and the algorithms MOFPCA~\cite{Pelegrina2021} and FairPCA~\cite{Samadi2018}. We used the following datasets and sensitive attributes (the first two were also used in~\cite{Samadi2018,Pelegrina2021}):
\begin{itemize}
    \item Taiwanese Credit Default (TCRED)\footnote{\url{https://archive.ics.uci.edu/ml/datasets/default+of+credit+card+clients}.}~\cite{Yeh2009}: Personal attributes and credit history information used to predict default payments in Taiwan. It consists of 30000 samples and 22 attributes. We considered as sensitive attribute the education level (higher and lower, with 24615 and 5385 samples, respectively).
    \item Labeled Faces in the Wild (LFW)\footnote{\url{http://vis-www.cs.umass.edu/lfw/}.}~\cite{Huang2008}: Public benchmark of photographs frequently used for face recognition. It consists of 13232 samples and 1764 attributes (36x49 pixels). We considered as sensitive attribute the gender (female and male, with 2962 and 10270 samples, respectively).
    \item The Law School Admissions Council's (LSAC)\footnote{\url{http://www.seaphe.org/databases.php}.}~\cite{Wightman1998}: A dataset collected from law school students that investigates ethical concerns in the bar passage rates. It consists of 26551 samples and 10 attributes. We considered as sensitive attribute the race (black and white + other, with 1790 and 24761 samples, respectively).
\end{itemize}
    
For each one, we collected the predictive attributes and defined the sensitive one, which will be only used to split the dataset into groups $A$ and $B$. Therefore, we do not consider the sensitive attribute in the dimensionality reduction task.

We also evaluate the considered methods in scenarios in which there are unbalanced and balanced datasets with respect to the sensitive attributes. The obtained results are presented in the sequel.

\subsection{Fair dimensionality reduction: Unbalanced datasets}
\label{subsec:exp2}

We first addressed scenarios with unbalanced datasets, i.e., when $n_A \neq n_B$. Figure~\ref{fig:exp_ovre_all} presents the reconstruction errors and the fairness measures for different reduced dimensions. In all cases, the u-FPCA approach led to the lower values of fairness measure. However, since this approach allows an increase of the reconstruction error of both groups in order to achieve fairness, in some datasets we obtained higher overall reconstruction errors in comparison with MOFPCA and FairPCA (see Figure~\ref{fig:exp_ovre}). However, if we consider the c-FPCA approach, we achieve good values of fairness with a small loss in the overall reconstruction error in comparison with the u-FPCA approach. In other words, if we compare the reconstruction errors of groups $A$ and $B$ in different reduced dimensions, the c-FPCA approach could reduce the disparities between these values without damaging the overall reconstruction error. Therefore, by comparing with the classical PCA, we achieved reconstruction errors such that $\bar{\mathcal{R}}_A(\mathbf{U}), \bar{\mathcal{R}}_B(\mathbf{U}) \leq \bar{\mathcal{R}}_B^{PCA}(\tilde{\mathbf{U}})$, where group $B$ is the unprivileged group. Note, in Figure~\ref{fig:exp_re}, that this is not always true for the u-FPCA approach.

\begin{figure*}[h!t]
\centering
\subfloat[]{\includegraphics[width=6.3in]{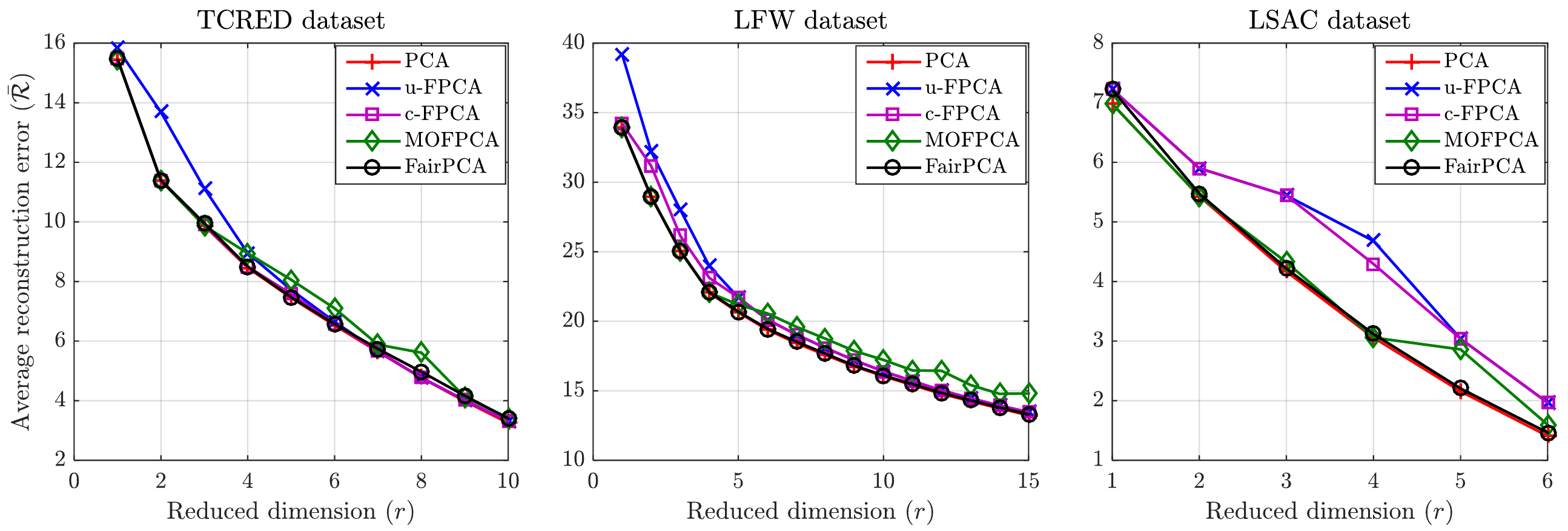}
\label{fig:exp_ovre}}
\hfil
\subfloat[]{\includegraphics[width=6.3in]{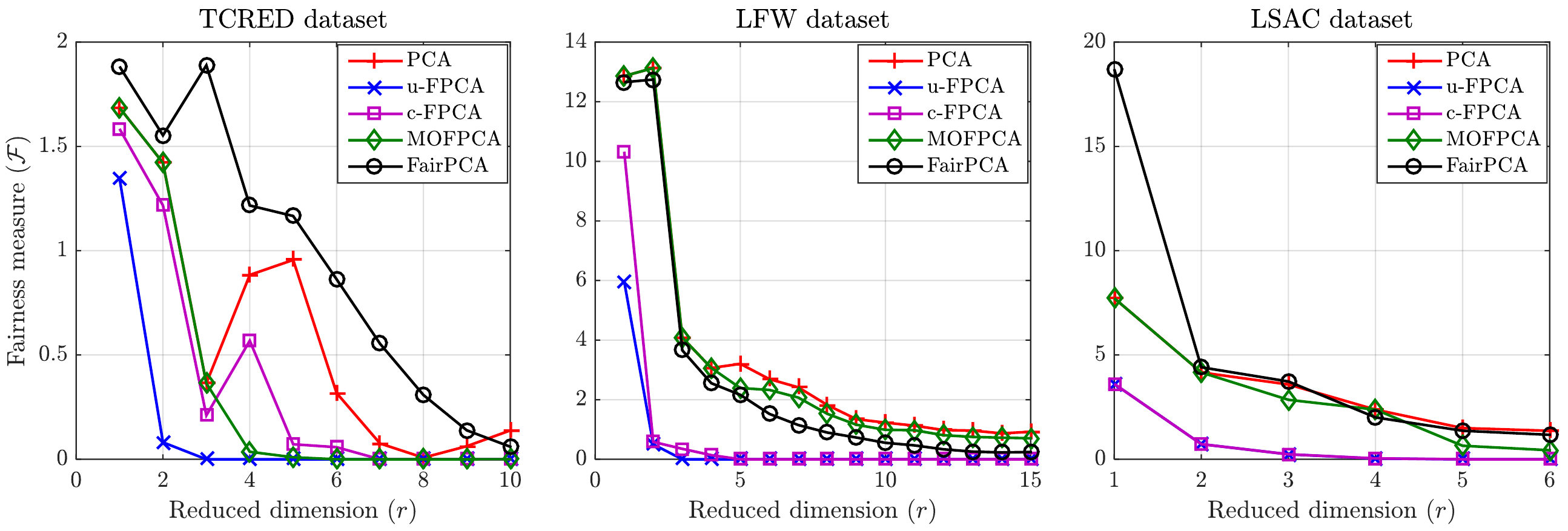}
\label{fig:exp_fm}}
\caption{A comparison between the reconstruction error and the fairness measure in unbalanced datasets.~\ref{fig:exp_ovre} Reconstruction error.~\ref{fig:exp_fm} Fairness measure.}
\label{fig:exp_ovre_all}
\end{figure*}

Although the c-FPCA approach led to interesting results in fair dimension reduction problems, there was a specific scenario in which the MOFPCA achieved better results in terms of fairness. In the TCRED dataset and for a dimensionality reduction into 4, 5 and 6-dimensional samples, the MOFPCA practically eliminated the disparity between the sensitive groups. However, as can be seen in Figure~\ref{fig:exp_re_credit}, it payed the same price as the u-FPCA: the reconstruction error of both groups increased. As we have already mentioned, this can be a problem in practical applications, and a projection with a better compromise between the objectives, such as the one provided by the c-FPCA proposed approach, could be a good solution for the dimensionality reduction problem. With respect to the FairPCA algorithm, most results are close (or even worst) to the classical PCA (see Figure~\ref{fig:exp_ovre_all} ). However, it is important to recall that the fairness measure in FairPCA is different from the one adopted in this paper.

\begin{figure*}[h!t]
\centering
\subfloat[]{\includegraphics[width=6.3in]{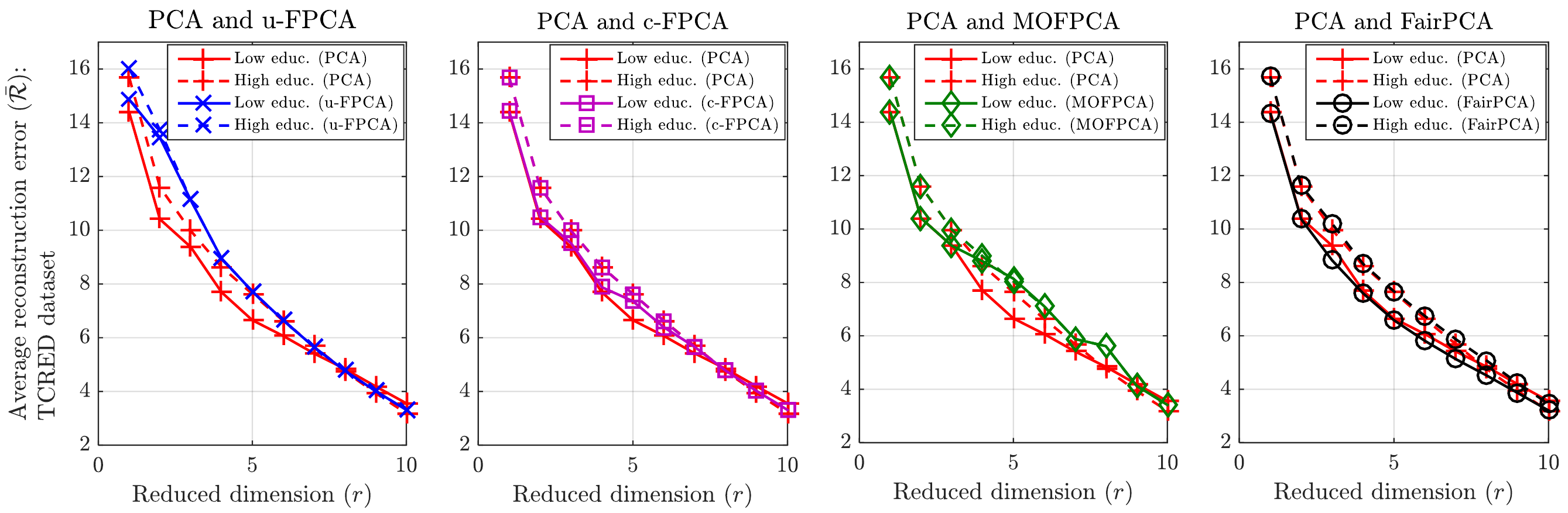}
\label{fig:exp_re_credit}}
\hfil
\subfloat[]{\includegraphics[width=6.3in]{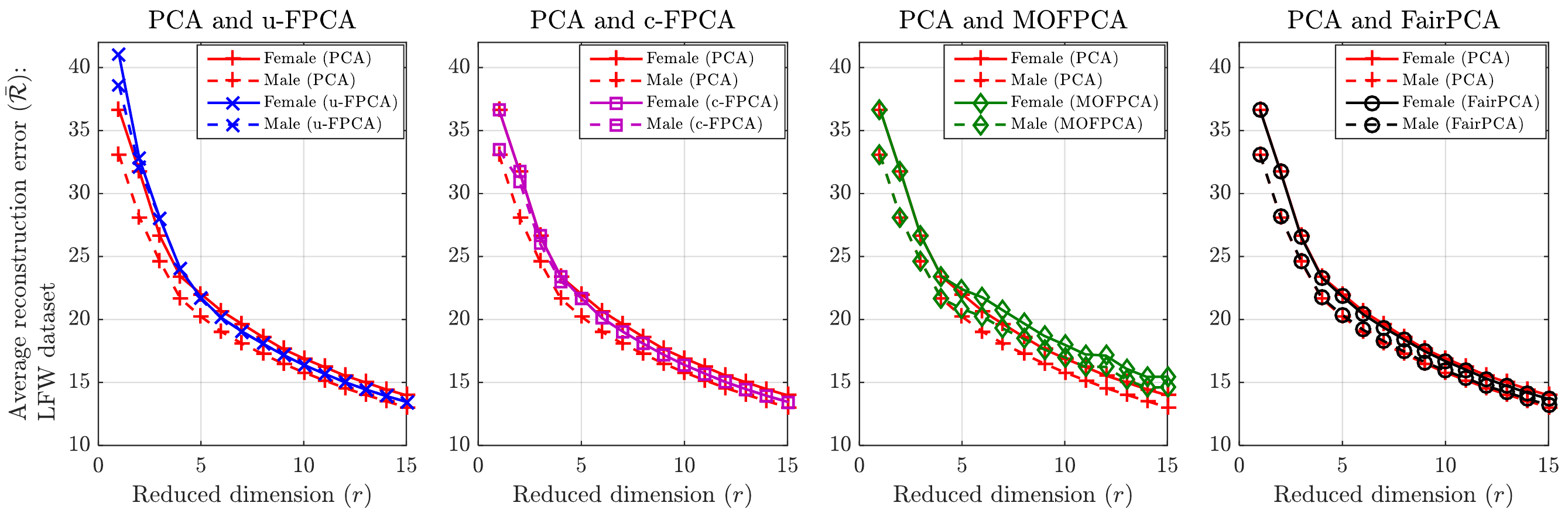}
\label{fig:exp_re_lfw}}
\hfil
\subfloat[]{\includegraphics[width=6.3in]{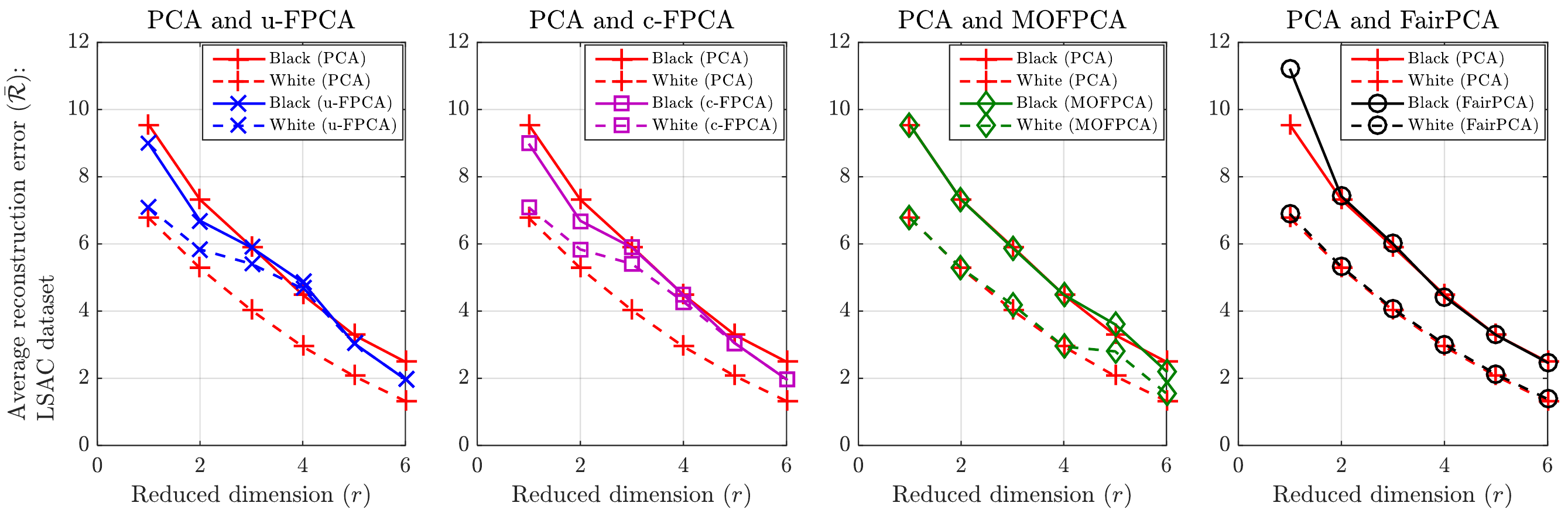}
\label{fig:exp_re_lsac_bw}}
\caption{Reconstruction errors of each sensitive group in unbalanced datasets.~\ref{fig:exp_re_credit} TCRED dataset.~\ref{fig:exp_re_lfw} LFW dataset.~\ref{fig:exp_re_lsac_bw} LSAC dataset.}
\label{fig:exp_re}
\end{figure*}

\subsection{Fair dimensionality reduction: Balanced datasets}
\label{subsec:exp3}

In the previous experiment, we considered unbalanced datasets and showed the interesting results provided by the proposed c-FPCA approach. However, since some bias in machine learning comes from unbalanced datasets, there may be a concern if our proposals only works on such scenarios. Motivated by this question, in this experiment, we verify the consistency of our proposals in balanced scenarios for each dataset. For this purpose, we considered the same datasets and selected a subset of samples\footnote{We selected the first $n_{min}$ samples of each group, where $n_{min}$ is the minimum between $n_A$ and $n_B$.} such that $n_A = n_B$. Therefore, there will not be an interference of unbalanced data in the dimension reduction problem.

Figures~\ref{fig:exp_ovre_eq_all} and~\ref{fig:exp_re_eq} present the obtained results. Similarly as in the previous experiment, the u-FPCA approach achieved very good results in terms of fairness. However, there were some considerable loss in the overall reconstruction error. Although it reduced the disparity, the reconstruction error of both groups increased (see the u-FPCA results in Figure~\ref{fig:exp_re_eq}). Conversely, for all datasets, the c-FPCA approach could improve fairness with a very small loss in the overall reconstruction error. Therefore, it was consistent with the previous experiment.

\begin{figure*}[h!t]
\centering
\subfloat[]{\includegraphics[width=6.3in]{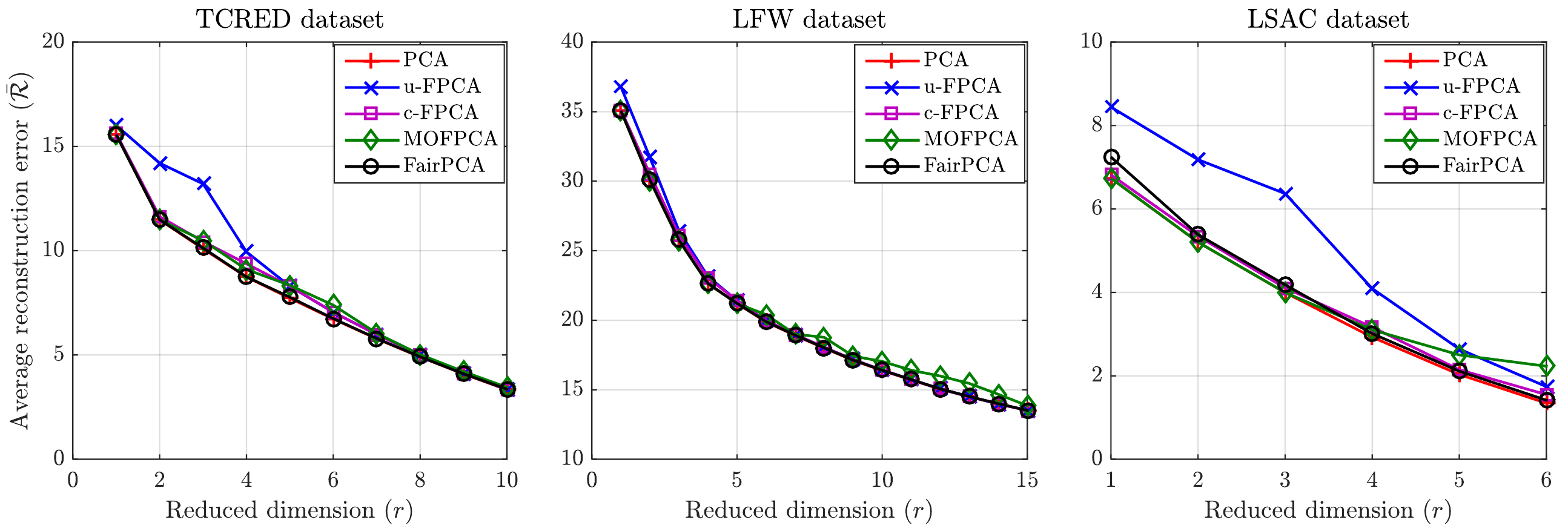}
\label{fig:exp_ovre_eq}}
\hfil
\subfloat[]{\includegraphics[width=6.3in]{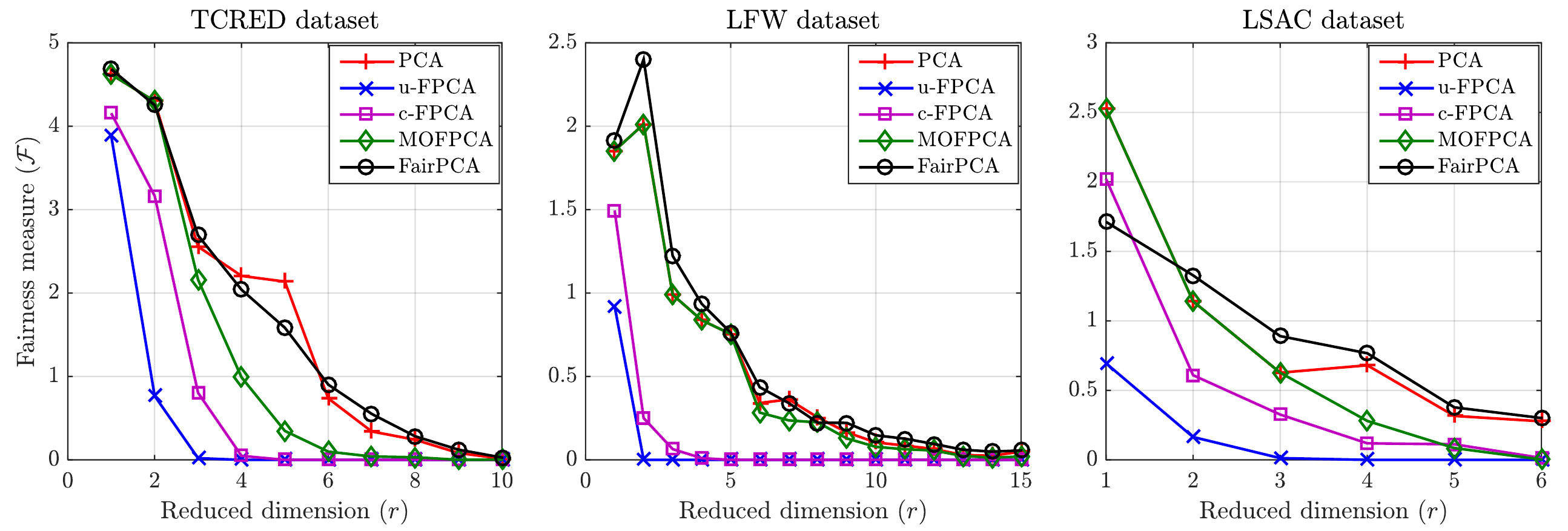}
\label{fig:exp_fm_eq}}
\caption{A comparison between the reconstruction error and the fairness measure in balanced datasets.~\ref{fig:exp_ovre_eq} Reconstruction error.~\ref{fig:exp_fm_eq} Fairness measure.}
\label{fig:exp_ovre_eq_all}
\end{figure*}

\begin{figure*}[h!t]
\centering
\subfloat[]{\includegraphics[width=6.3in]{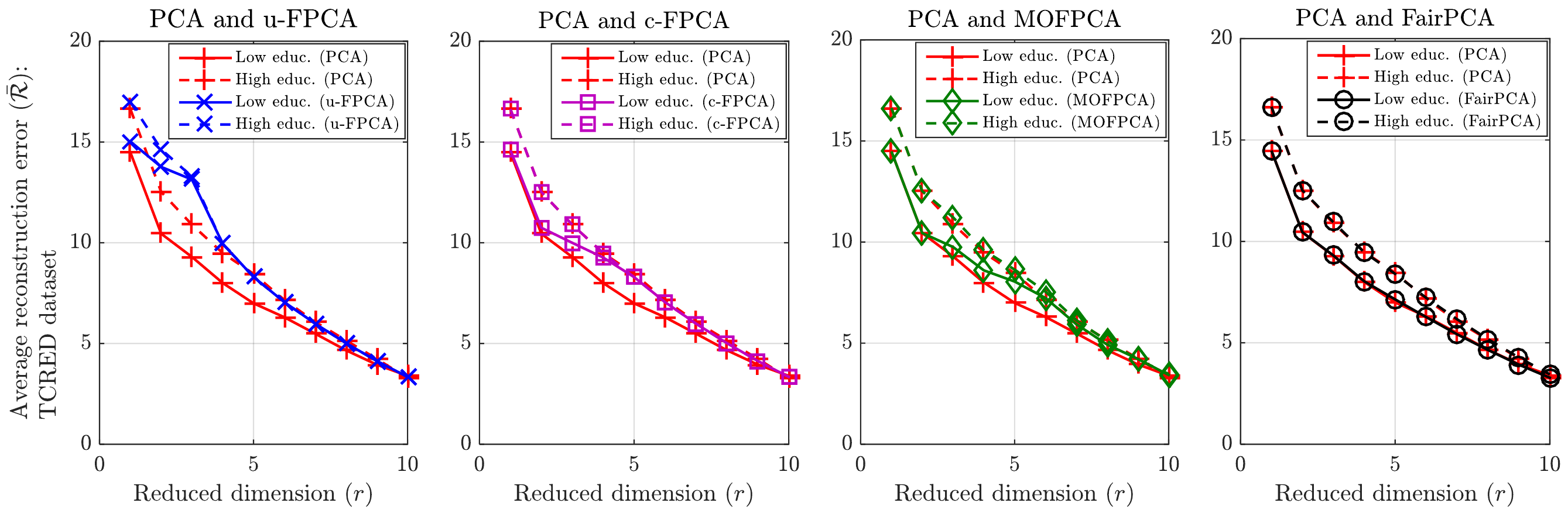}
\label{fig:exp_re_credit_eq}}
\hfil
\subfloat[]{\includegraphics[width=6.3in]{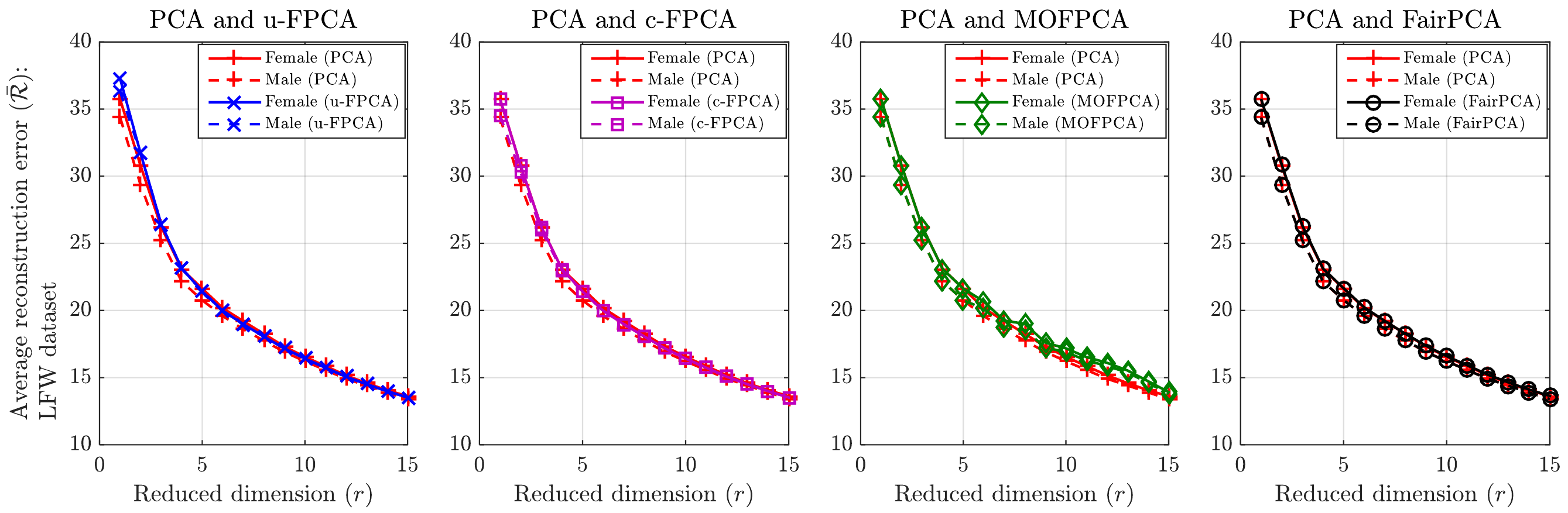}
\label{fig:exp_re_lfw_eq}}
\hfil
\subfloat[]{\includegraphics[width=6.3in]{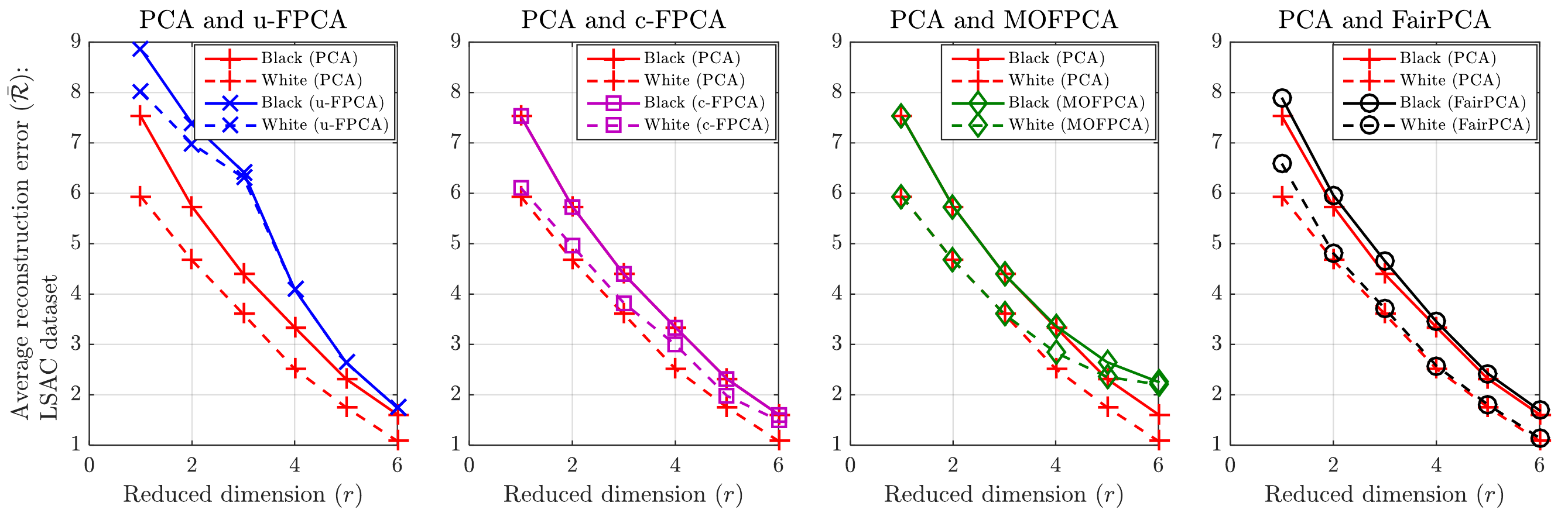}
\label{fig:exp_re_lsac_bw_eq}}
\caption{Reconstruction errors of each sensitive group in balanced datasets.~\ref{fig:exp_re_credit_eq} TCRED dataset.~\ref{fig:exp_re_lfw_eq} LFW dataset.~\ref{fig:exp_re_lsac_bw_eq} LSAC dataset.}
\label{fig:exp_re_eq}
\end{figure*}

If one analyzes the results provided by the MOFPCA, there was a slightly improvement in some reduced dimensions (see Figure~\ref{fig:exp_fm_eq}). However, the c-FPCA could further improve fairness with a lower loss in the overall reconstruction error. With respect to the FairPCA, the achievements were consistent with the previous experiment, with results very similar to the classical PCA, or even worst in some scenarios (see the TCRED and the LSAC datasets in Figure~\ref{fig:exp_ovre_eq_all}). We note such a similarity in the reconstruction errors of both sensitive groups, which are practically the same as in the application of the classical PCA (see the FairPCA results in Figure~\ref{fig:exp_re_eq}).

\section{Conclusion}
\label{sec:concl}

In this paper, we addressed the problem of fair dimensionality reduction based on principal component analysis. Since the classical PCA only minimizes the overall reconstruction error of a dataset, it was not conceived in order to avoid possible disparities between sensitive groups. As a consequence, the application of such a technique may lead to a reduced dataset in which a specific group is underrepresented with respect to another one. This may create (or even increase) social bias.

In order to maximize the information retained in the dimensionality reduction while mitigating disparities between sensitive groups, we formulate an optimization problem that exploits both overall reconstruction error and fairness measure when searching for the projection matrix. We also proved that the solution of such a formulation is as simple as the solution of the classical PCA, which consists of an eigendecomposition. Moreover, we proposed two one-dimensional algorithms, which exploit eigendecomposition solutions, to achieve a fair dimensionality reduction. Therefore, our proposal can be easily deployed in any already running systems.

The experimental results in several real datasets attested that our proposal can find a projection matrix that minimizes the disparity between the sensitive groups without a large loss in the overall reconstruction error. Moreover, in contrast with other existing methods, our results were consistent in scenarios with both unbalanced and balanced data.

It is worth highlighting that this article addresses the fairness issue in unsupervised dimensionality reduction. However, ongoing works extend our proposal in supervised principal component analysis. In this context, the fair dimensionality reduction technique can find a projection matrix that improve the classification accuracy while reducing the disparity between the true/false positive/negative rates of sensitive groups.

\section*{Acknowledgements}

The authors would like to thank the grants \#2020/01089-9, \#2020/09838-0, \#2020/10572-5 and \#2021/11086-0, São Paulo Research Foundation (FAPESP), and the grant \#311357/2017-2, National Council for Scientific and Technological Development (CNPq), for the financial support.

\appendix

\section*{Eigenvectors as a solution for PCA}
\label{app:pca_sol}

In this section, we demonstrate that the solution of PCA consists of the eigenvectors associated with the highest eigenvalues of the covariance matrix of $\mathbf{X}$. We follow an iterative approach~\cite{Jolliffe2002}, which starts by searching for an unitary projection vector $\mathbf{u}_1$ that maximizes the variance of the projected data $\tilde{\mathbf{x}}_1 = \mathbf{X}\mathbf{u}_1$. This variacne is given by $\text{Var}\left[\tilde{\mathbf{x}}_1 \right] = \frac{1}{n}\tilde{\mathbf{x}}_1^T\tilde{\mathbf{x}}_1 = \frac{1}{n} \mathbf{u}_1^T \mathbf{X}^T \mathbf{X}\mathbf{u}_1 = \mathbf{u}_1^T \mathbf{C}_{\mathbf{X}}\mathbf{u}_1$, where $\mathbf{C}_{\mathbf{X}}$ is the covariance of $\mathbf{X}$. The optimization problem is the following:
\begin{equation}
\label{eq:pca_pc1}
\begin{array}{ll}
\displaystyle\max_{\mathbf{u}_1} & \mathbf{u}_1^T\mathbf{C}_{\mathbf{X}}\mathbf{u}_1 \\
\text{s.t.} & \mathbf{u}_1^T\mathbf{u}_1 = 1, 
\end{array}
\end{equation}
This optimization problem can be easily solved by using the Lagrange multipliers~\cite{Vanderbei2014}, which leads to:
\begin{equation}
\label{eq:pca_pc1_lagr}
\begin{array}{ll}
\displaystyle\max_{\mathbf{u}_1, \lambda_1} & \mathbf{u}_1^T\mathbf{C}_{\mathbf{X}}\mathbf{u}_1 - \lambda_1 \left(\mathbf{u}_1^T\mathbf{u}_1 - 1 \right) \\
\end{array}.
\end{equation}
By taking the gradient of this cost function, one obtains that
\begin{equation}
2\mathbf{C}_{\mathbf{X}}\mathbf{u}_1 - 2\lambda_1 \mathbf{u}_1 = \mathbf{0} \rightarrow \mathbf{C}_{\mathbf{X}}\mathbf{u}_1 = \lambda_1 \mathbf{u}_1,
\end{equation}
where one recognizes that $\mathbf{u}_1$ is an eigenvector of $\mathbf{C}_{\mathbf{X}}$ and $\lambda_1$ is the associated eigenvalue. Therefore, the first projection vector is the eigenvector associated with the highest eigenvalue of the covariance matrix $\mathbf{C}_{\mathbf{X}}$. Moreover, since $\mathbf{C}_{\mathbf{X}}\mathbf{u}_1 = \lambda_1 \mathbf{u}_1$, $\text{Var}\left[\tilde{\mathbf{x}}_1 \right] = \mathbf{u}_1^T\mathbf{C}_{\mathbf{X}}\mathbf{u}_1 = \mathbf{u}_1^T\lambda_1 \mathbf{u}_1 = \lambda_1\mathbf{u}_1^T \mathbf{u}_1 = \lambda_1$,
i.e., the variance in $\tilde{\mathbf{x}}_1$ is given by the highest eigenvalue of $\mathbf{C}_{\mathbf{X}}$.

Once one has found the first principal component, one may move to the second one. Other than the constraint that ensures a unitary vector, one also needs to guarantee that the second principal component is orthogonal to the first one, i.e., $\mathbf{u}_2^T\mathbf{u}_1 = 0$. This leads to the following optimization problem:
\begin{equation}
\label{eq:pca_pc2}
\begin{array}{ll}
\displaystyle\max_{\mathbf{u}_2} & \mathbf{u}_2^T\mathbf{C}_{\mathbf{X}}\mathbf{u}_2 \\
\text{s.t.} & \mathbf{u}_2^T\mathbf{u}_2 = 1, \\
 & \mathbf{u}_2^T\mathbf{u}_1 = 0, \\
\end{array}
\end{equation}
One may also deal with~\eqref{eq:pca_pc2} by means of the Lagrange multipliers:
\begin{equation}
\label{eq:pca_pc2_lagr}
\begin{array}{ll}
\displaystyle\max_{\mathbf{u}_2, \lambda_2} & \mathbf{u}_2^T\mathbf{C}_{\mathbf{X}}\mathbf{u}_2 - \lambda_2 \left(\mathbf{u}_2^T\mathbf{u}_2 - 1 \right) - \phi \left(\mathbf{u}_2^T\mathbf{u}_1 \right)\\
\end{array}
\end{equation}
By taking the gradient of this cost function, one obtains that
\begin{equation}
\label{eq:grad_pca_1pc}
2\mathbf{C}_{\mathbf{X}}\mathbf{u}_2 - 2\lambda_2 \mathbf{u}_2 - \phi \mathbf{u}_1 = \mathbf{0}.
\end{equation}
If one multiplies the aforementioned equation on the left by $\mathbf{u}_1^T$ and since one assumed that $\mathbf{u}_2^T\mathbf{u}_1 = \mathbf{u}_1^T\mathbf{u}_2 = 0$, one obtains that
\begin{equation}
2\mathbf{u}_1^T\mathbf{C}_{\mathbf{X}}\mathbf{u}_2 - 2\lambda_2 \mathbf{u}_1^T\mathbf{u}_2 - \phi \mathbf{u}_1^T\mathbf{u}_1 = \mathbf{0} \rightarrow 2\mathbf{u}_1^T\mathbf{C}_{\mathbf{X}}\mathbf{u}_2 = \phi \mathbf{u}_1^T\mathbf{u}_1.
\end{equation}
Since we want that the projections $\tilde{\mathbf{x}}_1 = \mathbf{X}\mathbf{u}_1$ and $\tilde{\mathbf{x}}_2 = \mathbf{X}\mathbf{u}_2$ are uncorrelated, i.e., $\frac{1}{n}\mathbf{u}_1^T\mathbf{X}^T\mathbf{X}\mathbf{u}_2 = \mathbf{u}_1^T\mathbf{C}_{\mathbf{X}}\mathbf{u}_2 = 0$, this implies that $\phi \mathbf{u}_1^T\mathbf{u}_1 = 0$ and, therefore, $\phi$ must be equal to zero. One may also verify this condition on $\phi$ by taking~\eqref{eq:grad_pca_1pc} and, since $\mathbf{u}_1^T\mathbf{C}_{\mathbf{X}}\mathbf{u}_2 = \mathbf{u}_2^T\mathbf{C}_{\mathbf{X}}\mathbf{u}_1 = \mathbf{u}_2^T \lambda_1 \mathbf{u}_1 = \lambda_1 \mathbf{u}_2^T \mathbf{u}_1 = 0$, $\phi = 0$. This leads to the following expression:
\begin{equation}
2\mathbf{C}_{\mathbf{X}}\mathbf{u}_2 - 2\lambda_2 \mathbf{u}_2 = \mathbf{0} \rightarrow \mathbf{C}_{\mathbf{X}}\mathbf{u}_2 = \lambda_2 \mathbf{u}_2,
\end{equation}
where one also recognizes that $\mathbf{u}_2$ is an eigenvector of $\mathbf{C}_{\mathbf{X}}$ and $\lambda_2$ is the associated eigenvalue. In other words, the second principal component coefficients are the eigenvector associated with the second highest eigenvalue of $\mathbf{C}_{\mathbf{X}}$. Moreover, the variance in the projected data $\mathbf{\tilde{x}}_2$ is also equivalent to the second highest eigenvalue of $\mathbf{C}_{\mathbf{X}}$.

For the next $r-2$ principal components, the aforementioned technique can be iteratively used to find the projection matrix $\mathbf{U}$ whose columns are composed by the eigenvectors $\mathbf{u}_1, \mathbf{u}_2, \ldots, \mathbf{u}_r$ of $\mathbf{C}_{\mathbf{X}}$.

\bibliographystyle{unsrt}

\bibliography{_references}

\end{document}